\documentclass{article}

\PassOptionsToPackage{numbers, compress, sort}{natbib}

\usepackage[final]{neurips_2022}

\usepackage[utf8]{inputenc} %
\usepackage[T1]{fontenc}    %
\usepackage{hyperref}       %
\usepackage{url}            %
\usepackage{booktabs}       %
\usepackage{amsfonts}       %
\usepackage{nicefrac}       %
\usepackage{microtype}      %
\usepackage{xcolor}         %

\usepackage{amsmath,amsfonts,bm}

\def\eqref#1{equation~\ref{#1}}

\def\1{\bm{1}}

\DeclareMathAlphabet{\mathsfit}{\encodingdefault}{\sfdefault}{m}{sl}
\SetMathAlphabet{\mathsfit}{bold}{\encodingdefault}{\sfdefault}{bx}{n}

\usepackage[acronym,smallcaps,nowarn,section,nogroupskip,nonumberlist,nohypertypes={acronym,notation}]{glossaries}
\usepackage[labelformat=simple]{subcaption}

\usepackage{accents}
\usepackage{cleveref}
\crefformat{equation}{(#2#1#3)}
\usepackage{dblfloatfix} 

\usepackage{amsthm}
\usepackage{thmtools}
\usepackage{thm-restate}
\declaretheorem[name=Proposition]{proposition}
\declaretheorem[name=Corollary]{corollary}

\usepackage{wrapfig}

\usepackage{dsfont}  %

\newcommand{\mbf}[1]{\mathbf{#1}}

\newcommand{\xb}{\mbf{x}}
\newcommand{\zb}{\mbf{z}}
\newcommand{\thetab}{\boldsymbol{\theta}}
\newcommand{\thetashb}{\boldsymbol{\theta_\parallel}}
\newcommand{\thetaspb}{\boldsymbol{\theta_\perp}}
\newcommand{\gb}{\mbf{g}}
\newcommand{\alphab}{\bm{\alpha}}

\newcommand{\conv}{\text{Conv}}
\newcommand{\aff}{\text{Aff}}
\usepackage{enumitem}
\usepackage{stmaryrd}
\usepackage{adjustbox}
\renewcommand{\bfseries}{\fontseries{b}\selectfont}
\newrobustcmd{\B}{\bfseries}

\usepackage{siunitx}
\newacronym{mlp}{MLP}{multi-layer perceptron}
\newacronym{mdp}{MDP}{Markov Decision Process}
\newacronym{rl}{RL}{reinforcement learning}
\newacronym{smto}{SMTO}{Specialized Multi-Task Optimizer}
\newacronym{mgda}{MGDA}{Multiple-Gradient Descent Algorithm}
\newacronym{imtl}{IMTL}{Impartial Multi-Task Learning}
\newacronym{mtl}{MTL}{Multi-Task Learning}

\title{In Defense of the Unitary Scalarization\\for Deep Multi-Task Learning}

\newcommand{\equal}{\texttt{*}}

\newtoggle{appendix}
\toggletrue{appendix}

\author{%
	Vitaly Kurin\thanks{Equal contribution. \vspace{-10pt}} \\
	University of Oxford \\
	\texttt{vitaly.kurin@cs.ox.ac.uk} \\
	 \And
	Alessandro De Palma\footnotemark[1] \\
	University of Oxford \\
	\texttt{adepalma@robots.ox.ac.uk} \\ 
	 \AND  
	Ilya Kostrikov\\
	University of California, Berkeley\\ %
	New York University \\
	\And
	Shimon Whiteson \\
	University of Oxford \\
	\And
	M. Pawan Kumar \\
	University of Oxford 
}

\begin{document}

\maketitle

\vspace{-10pt}
\begin{abstract}

Recent multi-task learning research argues against \emph{unitary scalarization}, where training simply minimizes the sum of the task losses. Several ad-hoc multi-task optimization algorithms have instead been proposed, inspired by various hypotheses about what makes multi-task settings difficult. 
The majority of these optimizers require per-task gradients, and introduce significant memory, runtime, and implementation overhead.
We show that unitary scalarization, coupled with standard regularization and stabilization techniques from single-task learning, matches or improves upon the performance of complex multi-task optimizers in popular supervised and reinforcement learning settings.
We then present an analysis suggesting that many specialized multi-task optimizers can be partly interpreted as forms of regularization, potentially explaining our surprising results.
We believe our results call for a critical reevaluation of recent research in the area.
\end{abstract}

\vspace{-5pt}
\section{Introduction}
\gls{mtl}~\citep{Caruana1997} exploits similarities between tasks to yield models that are more accurate, generalize better and require less training data.
Owing to the success of \gls{mtl} on traditional machine learning models~\citep{Heskes2000,Bakker2003,Theodoros2004} and of deep single-task learning across a variety of domains, a growing body of research has focused on deep \gls{mtl}.
The most straightforward way to train a neural network for multiple tasks at once is to minimize the sum of per-task losses. 
Adopting terminology from multi-objective optimization, we call this approach \emph{unitary~scalarization}.

While some work shows that multi-task networks trained via unitary scalarization exhibit superior performance to independent per-task models~\citep{Kokkinos2017,DBLP:journals/corr/abs-2104-08212}, others suggest the opposite~\citep{Teh2017, Kendall2017, Sener2018}. 
As a result, many explanations for the difficulty of \gls{mtl} have been proposed, each motivating a new~\gls{smto}~\citep{Sener2018, Liu2021, Yu2020, Chen2020,Wang2021}. 
These works typically claim that the proposed~\gls{smto} outperforms unitary scalarization, in addition to relevant prior work.
However, \glspl{smto}~usually require access to per-task gradients either with respect to the shared parameters, or to the shared representation. Therefore, their reported performance gain comes at significant computation and memory cost, the overhead scaling linearly with the number of tasks. By contrast, unitary scalarization requires only the average of the gradients across tasks, which can be computed via a single~backpropagation. 

Existing \glspl{smto} were introduced to solve challenges~related to the optimization of the deep~\gls{mtl} problem.
We instead postulate that the reported weakness of unitary scalarization is linked to experimental variability or to a lack of regularization, leading to the following~contributions:
\vspace{-5pt}
\begin{itemize}[leftmargin=.5cm]
	\item A comprehensive experimental evaluation~($\S$\ref{sec:experiments}) of recent \gls{smto}s on popular multi-task benchmarks, showing that no \gls{smto} consistently outperforms unitary scalarization in spite of the added complexity and overhead. %
	In particular, either the differences between unitary scalarization and \glspl{smto} are not statistically significant, or they can be bridged by standard regularization and stabilization techniques from the single-task literature.
	Our \gls{rl} experiments include optimizers previously applied only to supervised learning. 
	\item An empirical and technical analysis of the considered \glspl{smto}, suggesting that they reduce overfitting on the multi-task problem and hence act as regularizers ($\S$\ref{sec:critical-analysis}). We conduct an ablation study and provide a collection of novel and existing technical results that support this hypothesis.
	\item Code to reproduce the experiments, including a unified PyTorch~\citep{Paszke2019} implementation of the considered \glspl{smto}, is available at \url{https://github.com/yobibyte/unitary-scalarization-dmtl}.
\end{itemize}
\vspace{-5pt}
We believe that our results suggest that the considered \glspl{smto} can be often replaced by less expensive techniques.
We hope that these surprising results stimulate the search for a deeper understanding~of~MTL.

\section{Related Work}
\vspace{-2pt}

Before diving into details of specific~\glspl{smto} in Section~\ref{sec:critical-analysis}, we provide a high-level overview of the deep \gls{mtl} research.
Seminal work in \gls{mtl} includes \emph{hard parameter sharing}~\citep{caruanaThesis}: sharing neural network parameters between all tasks with, possibly, a separate part of the model for each task.
Hard parameter sharing is still the major \gls{mtl} approach adopted in natural language processing~\citep{DBLP:conf/icml/CollobertW08, DBLP:journals/corr/abs-2109-09138}, computer vision~\citep{DBLP:conf/cvpr/MisraSGH16}, and speech recognition~\citep{DBLP:conf/icassp/SeltzerD13}.
In this work, we implicitly assume that each parameter update employs information from all tasks.
However, not all works satisfy this assumption, either due to a large number of tasks~\citep{DBLP:conf/ijcai/CappartCK00V21, DBLP:conf/nips/KurinGWC20}, or simply as an implementation decision~\citep{DBLP:conf/icml/HuangMP20,DBLP:conf/iclr/KurinIRBW21}.
In this setting, \gls{mtl} resembles other problems dealing with multiple tasks, i.e., continual~\citep{khetarpal2020towards}, curriculum~\citep{DBLP:journals/jmlr/NarvekarPLSTS20}, and meta-learning~\citep{DBLP:journals/corr/abs-2004-05439}, which are not the focus of this work.

\looseness=-1
Many works strive to improve the performance of deep multi-task models.
One line of research hypothesizes that conflicting per-task gradient directions lead to suboptimal models, and focuses on explicitly removing such conflicts~\cite{Yu2020,Chen2020,Liu2021,Wang2021,javaloy2022rotograd,liu2021conflict}. 
Some authors postulate that loss imbalances across tasks hinder learning, proposing loss reweighting methods~\citep{Kendall2017,Chen2018,rlw}. \citet{Sener2018} and \citet{Navon2022} propose that tasks compete for model capacity and interpret \gls{mtl} as multi-objective optimization in order to cope with inter-task competition. 
Here, we focus on algorithms that explicitly rely on per-task gradients to try to outperform unitary scalarization ($\S$\ref{sec:critical-analysis}).
Research on multi-task architectures~\citep{DBLP:conf/cvpr/MisraSGH16,Guo2020} or \gls{mtl} algorithms exclusively motivated by deterministic loss reweighting~\citep{Kendall2017,Guo2018,Liu2019} are orthogonal to our work.
Both topics are investigated by a recent survey on pixel-level multi-task computer vision problems~\citep{vandenhende2021multi}, which found that the minimization of tuned weighted sums of losses (scalarizations) is empirically competitive with deterministic loss reweighting and MGDA in the considered~settings.
These results are extended to popular \glspl{smto} by a critical review from \citet{Derrick2022}, concurrent to our work, which argues that the optimization and generalization performance of \glspl{smto} can be matched by tuning scalarization coefficients. 
Our work reaches a similar conclusion, demonstrating that unitary scalarization performs on par with \glspl{smto} when coupled with standard and inexpensive regularization or stabilization techniques.
In other words, \citet{Derrick2022} provide complementary support for the link between \glspl{smto} and regularization by showing that tuning scalarization weights positively affects generalization.

\looseness=-1
In addition to the common supervised settings, we also consider 
multi-task~\gls{rl}, whose research can be grouped into three categories:
the first adds auxiliary tasks providing additional inductive biases to speed up learning~\citep{DBLP:conf/iclr/JaderbergMCSLSK17} on a target task. 
The second, based on policy distillation, 
uses per-task teacher models to provide labels for a multi-task model or per-task policies as regularizers~\citep{DBLP:journals/corr/RusuCGDKPMKH15,DBLP:journals/corr/ParisottoBS15,DBLP:conf/nips/TehBCQKHHP17}.
The third directly learns a shared policy~\citep{DBLP:journals/corr/abs-2104-08212}, possibly via an \gls{smto}~\citep{Yu2020}. 
We focus on the third category, whose literature reports varying performance for unitary scalarization (better~\citep{DBLP:journals/corr/abs-2104-08212} or worse~\citep{Yu2020} than per-task models), indicating confounding factors in evaluation pipelines and further motivating our work. 
PopArt~\citep{van2016learning,DBLP:conf/aaai/HesselSE0SH19} performs scale-invariant value function updates in order to address differences in returns across environments, showing improvements in the multi-task setting while still using unitary scalarization.
PopArt does not require per-task gradients but introduces additional hyperparameters.
In our work, we address the differences in rewards by normalizing them at the replay buffer level. However, we believe both unitary scalarization and~\glspl{smto} might equally benefit from~PopArt.

\section{Multi-Task Learning Optimizers} \label{sec:setting}

We will now describe the deep \gls{mtl} training problem and popular algorithms employed for its solution.
Let $(X, Y) \in \mathbb{R}^{d \times n} \times \mathbb{R}^{o \times n}$ be the training set, composed of $n$ $d$-dimensional points and $o$-dimensional labels. In addition, $\mathcal{L}_i : \mathbb{R}^{o \times n} \times \mathbb{R}^{o \times n} \rightarrow \mathbb{R}$ denotes the loss for the $i$-th task, $\thetab \in \mathbb{R}^{S}$ the parameter space, $\mathcal{T}:= \{1, \dots, m\}$ the set of $m$ tasks. 
The goal of \gls{mtl} is to learn a single (generally task-aware) parametrized model $f : \mathbb{R}^{S} \times \mathbb{R}^{d \times n} \times \mathcal{T} \rightarrow \mathbb{R}^{o \times n}$ that performs well on all tasks $\mathcal{T}$. 
The parameter space is often split into a set of shared parameters across tasks (generally the majority of the architecture), denoted~$\thetashb$, and (possibly empty) task-specific parameters, denoted $\thetaspb$, so that $\thetab := [\thetashb, \thetaspb]^T$.
In this context, the model $f$ often takes on an encoder-decoder architecture, where the encoder $g$ learns a shared representation across tasks, and the decoders $h_i$ are task-specific predictive heads: $f(\thetab, X, i) = h_i (g(\thetashb, X), \thetaspb)$. In this case, we denote by $\zb = g(\thetashb, X) \in \mathbb{R}^{r \times n}$ the $r$-dimensional shared representation of $X$.

The training problem for \gls{mtl} is typically formulated as the sum of the per-task losses~\citep{Sener2018,Yu2020,Chen2020}:
\begin{equation}
	\label{eq:unit-scalarization}
	\smash{\min_{\thetab}} \left[\begin{array}{l} \mathcal{L}^{\text{MT}}(\thetab) := \sum_{i \in \mathcal{T}} \mathcal{L}_i (f(\thetab, X, i), Y)	\end{array}\right].
\end{equation}

\paragraph{Unitary Scalarization}
\looseness=-1
The obvious way to minimize the multi-task training objective in equation \cref{eq:unit-scalarization} is to rely on a standard gradient-based algorithm. While, for simplicity, we focus on standard gradient descent rather than mini-batch stochastic gradient descent, the notation can be adapted by replacing the dataset size $n$ by the mini-batch~size $b$. 
Equation \cref{eq:unit-scalarization} corresponds to a linear scalarization with unitary weights under a multi-objective interpretation of~\gls{mtl}; hence, we call the direct application of gradient descent on equation \cref{eq:unit-scalarization} \emph{unitary scalarization}.
 For vanilla gradient descent, this corresponds to taking a step in the opposite direction as the one given by the sum of per-task~gradients: $\nabla_{\thetab} \mathcal{L}^{\text{MT}} = \sum_{i \in \mathcal{T}}  \nabla_{\thetab} \mathcal{L}_i$. 
Per-task gradients are not required, as it suffices to directly compute the gradient of the sum $\mathcal{L}^{\text{MT}}$. Hence, when relying on deep learning frameworks based on reverse-mode differentiation, such as PyTorch~\citep{Paszke2019}, the backward pass is performed once per iteration (rather than $m$ times). Furthermore, the memory cost is a factor $m$ less than most \gls{smto}s, which require access to each $\nabla_{\thetab} \mathcal{L}_i$.
As a consequence, unitary scalarization is simple, fast, and memory efficient. 
Our experiments demonstrate that, when possibly coupled with single-task regularization such as early stopping, $\ell_2$ penalty or dropout layers~\citep{Srivastava2014}, this simple optimizer is strongly competitive with~\glspl{smto}.

\paragraph{MGDA}
\citet{Sener2018} point out that equation \cref{eq:unit-scalarization} can be cast as a multi-objective optimization problem with the following objective: $\boldsymbol{\mathcal{L}}^{\text{MT}}(\thetab) := [\mathcal{L}_1(\thetab), \dots, \mathcal{L}_m(\thetab)]^T$. 
A commonly employed solution concept in multi-objective optimization is Pareto optimality. A point $\thetab^*$ is called Pareto-optimal if, for any another point $\thetab^\dagger$ such that $\exists i \in \mathcal{T} : \mathcal{L}_i(\thetab^\dagger) < \mathcal{L}_i(\thetab^*)$, then $\exists j \in \mathcal{T} : \mathcal{L}_j(\thetab^\dagger) > \mathcal{L}_j(\thetab^*)$.
A necessary condition for Pareto optimality at a point is Pareto stationarity, defined as the lack of a shared descent direction across all losses at that point.
\citet{Sener2018} rely on \gls{mgda}~\citep{Desideri2012} to reach a Pareto-stationary point for shared parameters $\thetashb$. 
Intuitively, \gls{mgda} proceeds by repeatedly stepping in a shared descent direction~\citep{Fliege2000,Desideri2012}, which can be found by solving the following optimization~problem:
\begin{equation}
	\min_{\gb, \epsilon} \left[\epsilon + \nicefrac{1}{2} \left\lVert \gb \right \rVert_2^2\right] \quad
	\text{s.t. } \  \nabla_{\thetashb} \mathcal{L}_i ^T \gb \leq \epsilon \quad \forall \ i \in \mathcal{T},
	\label{eq:primal-mgda}
\end{equation}
whose dual takes the following form (corresponding to the formulation from \citet{Desideri2012}):
\begin{equation}
	\max_{\alphab \geq 0 } -\nicefrac{1}{2} \left\lVert \gb \right \rVert_2^2 \quad \text{s.t. } \ \sum_i \alpha_i \nabla_{\thetashb} \mathcal{L}_i  = -\gb, \quad \sum_{i \in \mathcal{T}} \alpha_i =  1.
	\label{eq:dual-mgda}
\end{equation}
In other words, \gls{mgda} takes a step in a direction $\gb$ given by the negative convex combination of per-task gradients, whose coefficients are given by solving equation \cref{eq:dual-mgda}.
In practice, per-task gradients are rescaled before applying \gls{mgda}: the original authors' implementation~\citep{Sener2018} relies on $\nabla_{\thetashb} \mathcal{L}_i \leftarrow \nicefrac{\nabla_{\thetashb} \mathcal{L}_i }{\left\lVert\nabla_{\thetashb} \mathcal{L}_i\right\rVert \mathcal{L}_i(\thetab)}$.
The convergence of \gls{mgda} to a Pareto-stationary point is still guaranteed after normalization~\citep{Desideri2012}.

\paragraph{IMTL}
\gls{imtl} \citep{Liu2021} is presented as an \gls{smto} that is not biased against any single task. It is composed of two complementary algorithmic blocks: IMTL-L, acting on task losses, and IMTL-G, acting on per-task~gradients.
IMTL-G follows the intuition that a multi-task optimizer should proceed along a direction $\gb = -\sum_i \alpha_i \nabla_{\thetashb} \mathcal{L}_i$ that equally represents per-task gradients. This is formulated analytically by requiring that the cosine similarity between $\gb$ and each $\nabla_{\thetashb} \mathcal{L}_i$ be the same.  To prevent the resulting problem from being underdetermined, \citet{Liu2021} add the constraint $\sum_{i \in \mathcal{T}} \alpha_i = 1$, resulting in a problem that admits a closed-form solution for $\gb$:
\begin{equation}
	\begin{array}{l}
	\gb^T \frac{ \nabla_{\thetashb} \mathcal{L}_1}{\left\lVert\nabla_{\thetashb} \mathcal{L}_1\right\rVert} = \gb^T \frac{\nabla_{\thetashb} \mathcal{L}_i}{\left\lVert\nabla_{\thetashb} \mathcal{L}_i\right\rVert} \enskip  \forall \ i \in \mathcal{T} \setminus \{1\}, \quad
	\gb = -\sum_i \alpha_i \nabla_{\thetashb} \mathcal{L}_i, \quad \sum_{i \in \mathcal{T}} \alpha_i = 1.	\end{array}
	\label{eq:imtlg-goal}
\end{equation}
IMTL-L, instead, aims to reweight task losses so that they are all constant over time, and equal to $1$. In order to limit oscillations of the scaling factors, the authors propose to learn them jointly with the network by minimizing a common objective via gradient descent. 
In particular, given $s_i \in \mathbb{R} \ \forall i \in \mathcal{T}$, \citet{Liu2021} derive the following form for the joint minimization problem: $\min_{\mbf{s}, \thetab} \left[\sum_i \left(e^{s_i} \mathcal{L}_i (\thetab)  - s_i\right) \right].$
As proved by \citet{Liu2021}, IMTL-L only has a rescaling effect on the update direction of IMTL-G. Unlike IMTL-G and  the other \glspl{smto} presented in this section, IMTL-L rescaling is designed to affect the updates for task-specific parameters $\thetaspb$ as well.

\paragraph{PCGrad} 
Let us write $\cos(\xb, \zb)$ for the cosine similarity between vectors $\xb$ and $\zb$.
\citet{Yu2020} postulate that multi-task convergence is severely slowed down if the following three conditions (named the \emph{tragic triad}) hold at once:
(i) conflicting gradient directions: $\cos(\nabla_{\thetashb} \mathcal{L}_i, \nabla_{\thetashb} \mathcal{L}_j ) < 0$ for some $i,j \in \mathcal{T}$;
(ii) differing gradient magnitudes: $\left\lVert\nabla_{\thetashb} \mathcal{L}_i\right\rVert \gg  \left\lVert\nabla_{\thetashb} \mathcal{L}_j\right\rVert $ for some $i,j \in \mathcal{T}$; and
(iii) the unitary scalarization $\mathcal{L}^{\text{MT}}$ has high curvature along $\nabla_{\thetashb} \mathcal{L}^{\text{MT}}$.
The PCGrad~\citep{Yu2020} \gls{smto} is presented as a solution to the tragic triad, targeted at the first condition. Consistent with the previous sections, let us denote the update direction by $\gb$. Furthermore, let $[\xb]_+ := \max(\xb, \mbf{0})$. Given per-task gradients $\nabla_{\thetashb} \mathcal{L}_i$, PCGrad iteratively projects each task gradient onto the normal plane of all the gradients with which it conflicts: \vspace{-5pt}
\begin{equation}
	\left[\begin{array}{l} \hspace{-4pt}
		\gb_i \leftarrow \nabla_{\thetashb} \mathcal{L}_i, \enskip \gb_i \leftarrow \gb_i + \hspace{-3pt} \left[\frac{-\gb_i^T \nabla_{\thetashb} \mathcal{L}_j (\xb)}{\left\lVert \nabla_{\thetashb} \mathcal{L}_j \right\rVert^2}\right]_+ \hspace{-7pt} \nabla_{\thetashb} \mathcal{L}_j \enskip \forall j \in \mathcal{T} \setminus \{i\}
	\end{array}\hspace{-5pt} \right] \forall i \in \mathcal{T}, \quad
	\gb =  -\sum_{i \in \mathcal{T}} \gb_i,
	\label{eq:pcgrad-update}
\end{equation}
where the iterative updates of $\gb_i$ with respect to $\nabla_{\thetashb}\mathcal{L}_j$ are performed in random order.

\paragraph{GradDrop}
\citet{Chen2020} focus on conflicting signs across task gradient entries, arguing that such conflicts lead to gradient ``tug-of-wars". 
The GradDrop \gls{smto}~\citep{Chen2020}, presented as a solution to this problem, proposes to randomly mask per-task gradients $\nabla_{\thetashb} \mathcal{L}_i$ so as to minimize such conflicts.
Specifically, GradDrop computes the  ``positive sign purity" $p_j$ for the task gradient's $j$-th entry
and then masks the $j$-th entry of each per-task gradient with probability increasing with $p_j$, if the entry is negative, or decreasing with $p_j$, if the entry is positive. Let us write $\mbf{p} := [p_1, \dots, p_S]$, where $S$ is the dimensionality of the parameter space (see $\S$\ref{sec:setting}), $\odot$ for the Hadamard product and $\mathds{1}_{\mbf{a}}$ for the indicator vector on condition $\mbf{a}$. Given a vector $\mbf{u}_i$, uniformly sampled in $[\mbf{0}, \mbf{1}]$ at each iteration, GradDrop takes a step in the direction given by:
\begin{equation}
	\gb = \sum_{i\in \mathcal{T}} 
	\left(\begin{array}{l}
		-\nabla_{\thetashb} \mathcal{L}_i \odot \mathds{1}_{ \left(\nabla_{\thetashb} \mathcal{L}_i > 0\right) } \odot \mathds{1}_{ \left(\mbf{u}_i > \mbf{p}\right)} \\
		-\ \nabla_{\thetashb} \mathcal{L}_i \odot\mathds{1}_{ \left(\nabla_{\thetashb} \mathcal{L}_i < 0 \right)} \odot \mathds{1}_{ \left(\mbf{u}_i < \mbf{p}\right)} 
	\end{array}\right)\text{, with } \enskip \mbf{p} = \frac{1}{2} \left( 1 + \frac{\sum_{i \in \mathcal{T}} \nabla_{\thetashb} \mathcal{L}_i}{\sum_{i \in \mathcal{T}} \left|\nabla_{\thetashb} \mathcal{L}_i\right|} \right).
	\label{eq:graddrop}
\end{equation}

\section{Experimental Evaluation} \label{sec:experiments}

Relying on a unified experimental pipeline, we present an empirical evaluation on common \gls{mtl} benchmarks of unitary scalarization ($\S$\ref{sec:setting}), of the popular~\glspl{smto} presented in $\S$\ref{sec:setting}, and of the recent RLW algorithms~\citep{rlw} due to their similarities with PCGrad and GradDrop (see $\S$\ref{sec:technical-reg}).
We benchmark against the two RLW instances that showed the best average performance in the original paper: RLW with weights sampled from a Dirichlet distribution (``RLW Diri.''), and RLW with weights sampled from a Normal distribution (``RLW Norm.''). 
The goal of this section is to assess the efficacy of a popular line of previous work, focusing on a few representative or well-established optimizers.
Therefore, we forego comparison with more recent \glspl{smto}~\citep{Navon2022,javaloy2022rotograd,liu2021conflict}. Nevertheless, we point out that these algorithms often lack significant enough improvements over the optimizers we consider, or may have substantial commonalities with them (see $\S$\ref{sec:technical-reg} for Nash-MTL~\citep{Navon2022}, which was published concurrently to the finalization of this work).
Whenever appropriate, we employ ``Unit.\ Scal.'' as shorthand for unitary scalarization.
We first present supervised learning experiments ($\S$\ref{sec:sl-experiments}), and then evaluate on a popular reinforcement learning benchmark~($\S$\ref{sec:rl-experiments}).

Our experiments indicate that the performance of unitary scalarization has been consistently underestimated in the literature.
By showing the variability between runs and by relying on standard regularization and stabilization techniques from the single-task literature, we demonstrate that \emph{no SMTO consistently outperforms unitary scalarization across the considered settings}. This result holds in spite of the added complexity and computational overhead associated with most \gls{smto}s. 
We provide a potential explanation of our results in $\S$\ref{sec:critical-analysis}.

\subsection{Supervised Learning} \label{sec:sl-experiments}

All the architectures employed in the supervised learning experiments conform to the encoder-decoder structure detailed in~$\S$\ref{sec:setting}.
Whenever suggested by the original authors for this context, the \gls{smto} implementations rely on per-task gradients with respect to the last shared activation, $\nabla_{\zb}$, rather than on the usually more expensive $\nabla_{\thetab} \mathcal{L}_i$. In particular, this is the case for MGDA, IMTL and GradDrop.
See appendix \ref{sec:supp-overview} for details concerning each individual~algorithm. 
Surprisingly, several \gls{mtl} works~\citep{Yu2020,Chen2020,Liu2021,rlw} report validation results, making it easier to overfit.
Instead, following standard machine learning practice, we select a model on the validation set, and later report test metrics for all benchmarks. 
Validation results are also available in appendix~\ref{sec:supp-sl-experiments}.
Appendix~\ref{sec:sl-setup} reports dataset descriptions, the computational setup, hyperparameter and tuning details.

\subsubsection{Multi-MNIST}
\begin{figure*}[t!]
	\begin{subfigure}{0.49\textwidth}
		\centering
		\includegraphics[width=\textwidth]{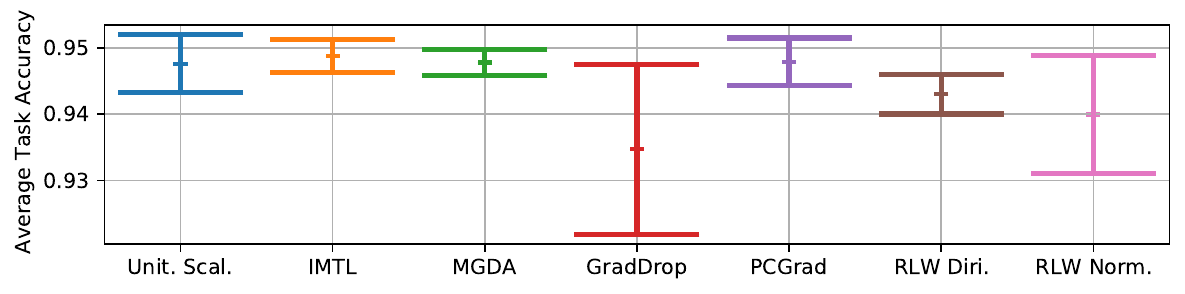}
		\vspace{-15pt}
		\caption{Avg. task test accuracy: mean and 95$\%$ CI (10~runs).}
		\label{fig:mnist-test}
	\end{subfigure}\hspace{5pt}
	\begin{subfigure}{0.49\textwidth}
		\centering
		\includegraphics[width=\textwidth]{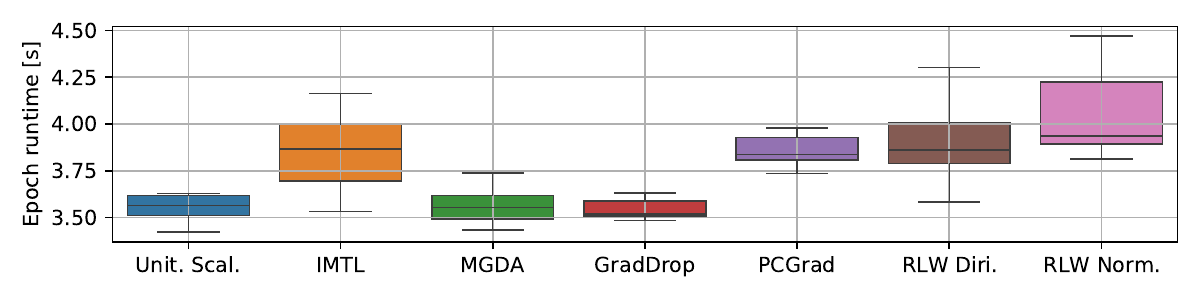}
		\vspace{-15pt}
		\caption{Box plots for the training time of an epoch (10 runs).}
		\label{fig:mnist-time}
	\end{subfigure}
	\vspace{-3pt}
	\caption{\label{fig:mnist} %
		No algorithm outperforms unitary scalarization on the Multi-MNIST dataset.}
	\vspace{-10pt}
\end{figure*}
\looseness=-1
We present results on the Multi-MNIST \citep{Sener2018} dataset, a simple two-task supervised learning benchmark. We employ a popular architecture from previous work~\citep{Sener2018,Yu2020} (see appendix~\ref{sec:sl-setup}), where a single dropout layer~\citep{Srivastava2014} (with dropout probability $0.5$) is employed in both the encoder and the decoder. $\ell_2$ regularization did not improve validation performance and was therefore omitted.
Figure~\ref{fig:mnist} reports the average task test accuracy, and the training time per epoch. 
For each run, the test model was selected as the model with the largest average task validation accuracy across the training epochs.
Appendix~\ref{sec:supp-sl-experiments} presents the results of Figure~\ref{fig:mnist} in tabular form, as well as the average task validation accuracy per epoch. 
As seen from the overlapping confidence intervals, none of the considered algorithms clearly outperforms the others. However, GradDrop displays  higher experimental variability. 
Finally, Figure~\ref{fig:mnist-time} shows that unitary scalarization also has among the lowest training times.
\subsubsection{CelebA} \label{sec:celeba-experiments}
\begin{figure*}[b!]
	\begin{subfigure}{0.49\textwidth}
		\centering
		\includegraphics[width=\textwidth]{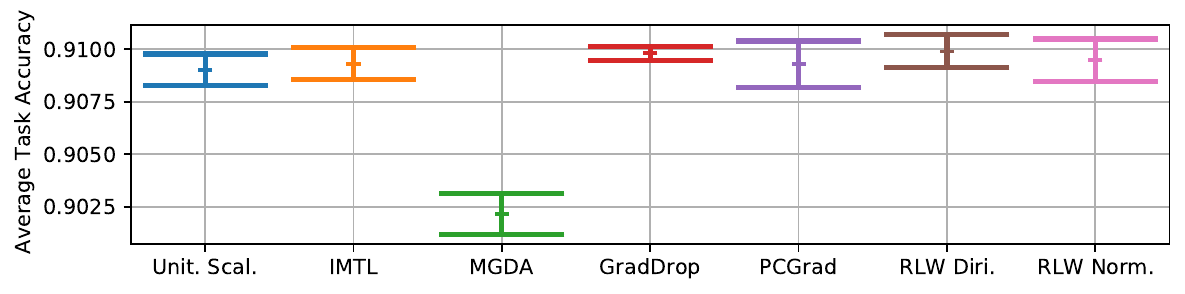}
		\vspace{-15pt}
		\caption{Avg. task test accuracy: mean and 95$\%$ CI (3 runs).}
		\label{fig:celeba-test}
	\end{subfigure}\hspace{5pt}
	\begin{subfigure}{0.49\textwidth}
		\centering
		\includegraphics[width=\textwidth]{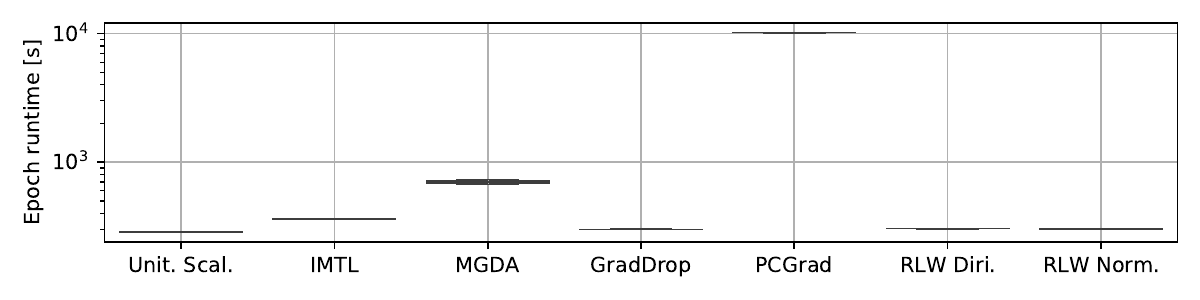}
		\vspace{-15pt}
		\caption{Box plots for the training time of an epoch (10 runs).}
		\label{fig:celeba-time}
	\end{subfigure}
	\vspace{-3pt}
	\caption{\label{fig:celeba}
	While \glspl{smto} display larger runtimes, none of them outperforms the unitary scalarization on the CelebA dataset.}
\end{figure*}
We now show results for the CelebA~\citep{liu2015faceattributes} dataset, a challenging $40$-task multi-label classification problem. We employ the same architecture as many previous studies~\citep{Sener2018,Yu2020,rlw,Liu2021} (see appendix~\ref{sec:sl-setup}).
We tuned $\ell_2$ regularization terms $\lambda$ for all \glspl{smto} in the following grid: $\lambda \in \{0, 10^{-4}, 10^{-3}\}$. 
The best validation performance was attained with $\lambda=10^{-3}$ for unitary scalarization, IMTL and PCGrad, and with $\lambda=10^{-4}$ for MGDA, GradDrop, and RLW. Validation performance was further stabilized by the addition of several dropout layers (see Figure~\ref{fig:celeba-unreg-validation}), with dropout probabilities from $0.25$ to $0.5$. 
We present an ablation study on the effect of regularization on this experiment in $\S$\ref{sec:reg-ablation}. Figure~\ref{fig:celeba-unreg-val-all} (appendix~\ref{sec:unreg}) shows that regularization improves the peak average validation performance for all the considered methods.
Analogously to our Multi-MNIST results, Figure~\ref{fig:celeba} plots the distribution of the training time per epoch, and the average test task accuracy. As with Multi-MNIST, the test model for each run was the one with maximal average validation task accuracy across epochs. 
In other words, if the peak is attained before the last epoch, we perform early stopping: as shown in Figure \ref{fig:celeba-plot} in appendix~\ref{sec:supp-sl-experiments} this is the case for most methods.
Due to the large number of tasks, Figure \ref{fig:celeba-time} shows relatively large runtime differences across methods. PCGrad is the slowest (roughly $35$ times slower than unitary scalarization). In fact, amongst the considered algorithms, it is the only one that computes per-task gradients over the parameters ($\nabla_{\thetab} \mathcal{L}_i \ \forall i \in \mathcal{T}$) at each iteration. GradDrop, MGDA and IMTL have overhead factors (compared to unitary scalarization) ranging from roughly $1.05$ to $2.4$ due to the relatively small size of $\zb$ for the employed architecture. The overhead of RLW is negligible: roughly $5\%$.
Nevertheless, due to largely overlapping confidence intervals in Figure \ref{fig:celeba-test}, none of the methods consistently outperforms unitary scalarization. In fact, owing to our adoption of explicit regularization techniques (see $\S$\ref{sec:reg-ablation}) its average performance is superior to that
reported in the literature~\citep{Sener2018,Liu2021}.

\subsubsection{Cityscapes} \label{sec:cityscapes-experiments}
\begin{figure}[t!]
	\begin{subfigure}{0.49\textwidth}
		\centering
		\includegraphics[width=\textwidth]{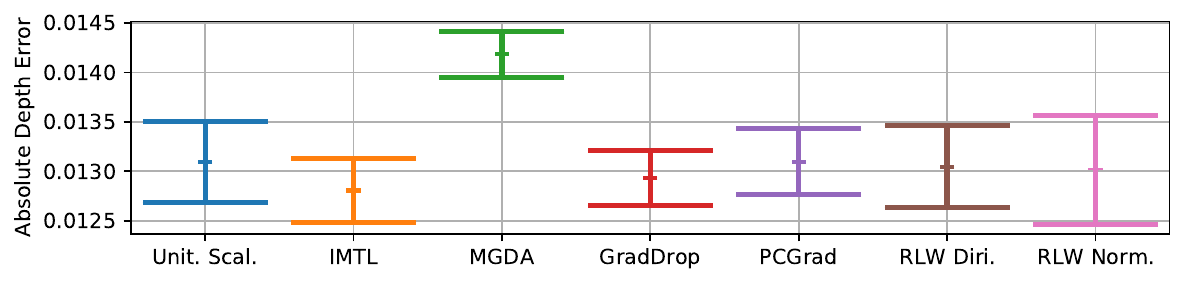}
		\vspace{-18pt}
		\caption{Absolute depth test error: lower is better.}
		\label{fig:cityscapes-best-absdepth}
	\end{subfigure}
	\begin{subfigure}{0.49\textwidth}
		\centering
		\includegraphics[width=\textwidth]{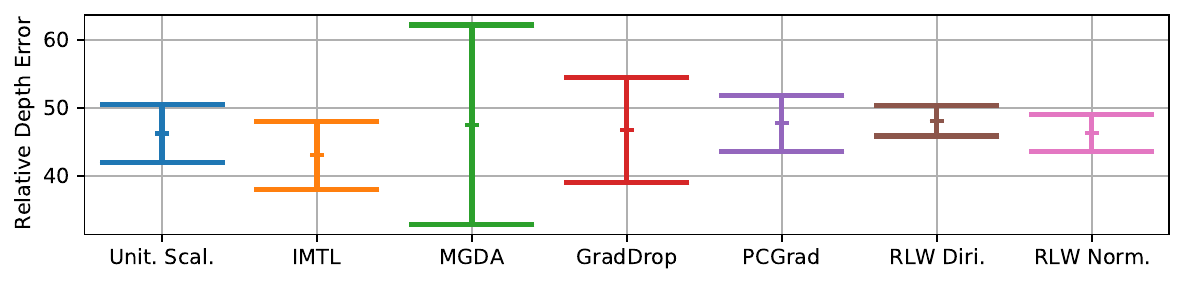}
		\vspace{-18pt}
		\caption{Relative depth test error: lower is better.}
		\label{fig:cityscapes-best-reldepth}
	\end{subfigure}
	\begin{subfigure}{0.49\textwidth}
		\centering
		\includegraphics[width=\textwidth]{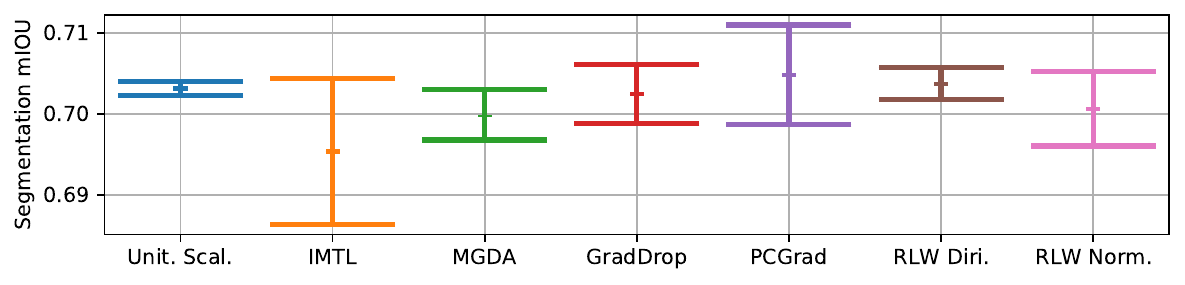}
		\vspace{-18pt}
		\caption{Test segmentation mIOU: higher is better.}
		\label{fig:cityscapes-best-segmIOU}
	\end{subfigure}
	\begin{subfigure}{0.49\textwidth}
		\centering
		\includegraphics[width=\textwidth]{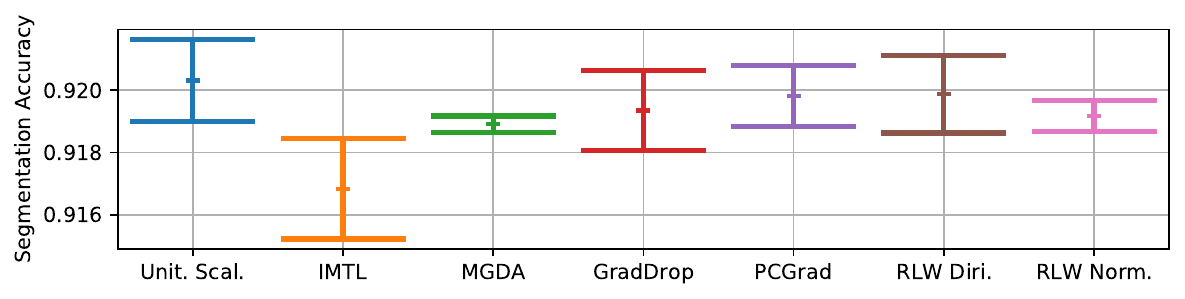}
		\vspace{-18pt}
		\caption{Test segmentation accuracy: higher is better.}
		\label{fig:cityscapes-best-segacc}
	\end{subfigure}
	\centering
	\begin{subfigure}{0.49\textwidth}
		\vspace{-10pt}
		\centering
		\includegraphics[width=\textwidth]{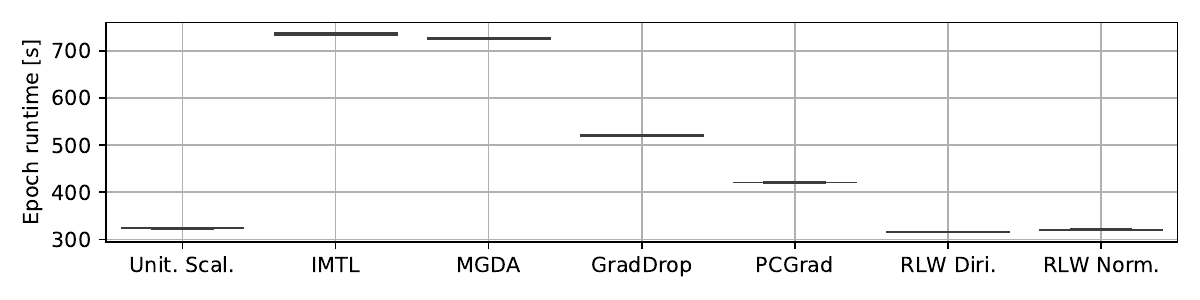}
		\vspace{-18pt}
		\caption{Box plots for the training time of an epoch (10 runs).}
		\label{fig:cityscapes-time}
	\end{subfigure}
	\hspace{13pt}
	\begin{minipage}{0.45\textwidth}
		\vspace{7pt}
		\caption{
			\label{fig:cityscapes} On Cityscapes, none of the \glspl{smto} outperforms unitary scalarization, which proves to be the most cost-effective algorithm.
			Subfigures \subref{fig:cityscapes-best-absdepth}-\subref{fig:cityscapes-best-segacc} report means for three runs, and their 95$\%$ CIs.
		}
	\end{minipage}
	\vspace{-20pt}
\end{figure}
In order to complement the multi-task classification experiments for Multi-MNIST and CelebA, we present results for Cityscapes~\citep{Cityscapes}, a dataset for semantic understanding of urban street scenes. We rely on a common encoder architecture from the literature~\citep{Liu2021,rlw} (see appendix~\ref{sec:sl-setup}),
with a single dropout layer in the task-specific heads~\citep{rlw}.
As for CelebA, unitary scalarization, IMTL, and PCGrad benefit from more regularization than the other optimizers: we employ $\lambda = 10^{-5}$ for these three algorithms, as it resulted in better validation performance on the majority of metrics, and $\lambda=0$ for the remaining methods.
Cityscapes is a heterogeneous \gls{mtl} problem: it contains tasks of different types whose validation metrics cannot be averaged to perform model selection.
Considering the lack of an established procedure in this context, we potentially evaluate a different model for each metric, chosen as the one with the best (maximal or minimal, depending on the metric) validation performance across epochs (we perform per-run early stopping). 
This procedure maximizes per-task performance, at the cost of increased inference time.
If inference time is a priority, an alternative model selection procedure could rely on relative task improvement~\citep{javaloy2022rotograd, Navon2022, liu2021conflict}, assuming that per-metric improvements are to be weighted linearly.
Nevertheless, any consistently applied model selection scheme serves the main goal of our work: evaluating all~\glspl{smto} on a fair ground.
Figure~\ref{fig:cityscapes} shows test results for two metrics per task, and the distribution of the training time per epoch. 
As with Multi-MNIST and CelebA, no training algorithm clearly outperforms unitary scalarization (significant overlaps across confidence intervals exist), which is again the least expensive method. 
In contrast with a popular hypothesis~\citep{Kendall2017,Chen2018,Liu2021}, this holds in spite of relatively large loss imbalances. In fact, the loss for the depth task is roughly $10$ times smaller than that of the segmentation task: see figures \ref{fig:cityscapes-plot-deploss}-\ref{fig:cityscapes-plot-segloss}.
Unlike CelebA (see Figure~\ref{fig:celeba-time}), IMTL, MGDA and GradDrop are significantly slower than unitary scalarization (factors from $1.6$ to $2.3$), due to the relatively (compared to the parameter space) large size of $\zb$ in the employed architecture.
PCGrad, instead, appears to be less expensive ($30\%$ more than the baseline), demonstrating the benefits of working on $\nabla_{\thetab} \mathcal{L}_i$ on this model.

\subsection{Reinforcement Learning}
\label{sec:rl-experiments}

For~\gls{rl} experiments, we use Meta-World~\citep{Yu2019} and the Soft Actor-Critic~\citep{DBLP:conf/icml/HaarnojaZAL18} implementation from~\citep{Sodhani2021}.
Unlike $\S$\ref{sec:sl-experiments}, the employed network architecture (see appendix \ref{sec:sl-setup}) is fully shared across tasks. Therefore, all \gls{smto} implementations for these experiments rely on per-task gradients with respect to network parameters $\nabla_{\thetab} \mathcal{L}_i$ (see $\S$\ref{sec:critical-analysis}).
Among the \glspl{smto} we consider, PCGrad is the only one developed with the~\gls{rl} setting in mind. 
For fairness and completeness, we add all the other~\glspl{smto} from the supervised learning experiments, and are the first to test these optimizers in the~\gls{rl} setting.
To stabilize learning, we increase the replay buffer size, a 
well known technique in single-task~\gls{rl}, add actor $l_2$ regularization, and modify the reward normalization employed by~\citet{Sodhani2021}. 
The unitary scalarization performance reported by \citet{Yu2020} is considerably lower than that of~\citet{Sodhani2021}, which we believe is due to the lack of reward normalization in the former.
\citet{Sodhani2021} keep a moving average of rewards in the environment, with a hyperparameter controlling the speed of the moving average.
As we show in Figure~\ref{fig:rewnorm-sensitivity}, the learning algorithm is sensitive to that hyperparameter.
Moreover, such normalization might make similar transitions have drastically different rewards stored in the replay buffer.
To alleviate these issues, we store the raw rewards in the buffer, and normalize only when a mini-batch is sampled.

\begin{figure*}[t!]
	\begin{subfigure}{0.49\textwidth}
		\centering
		\includegraphics[width=\textwidth]{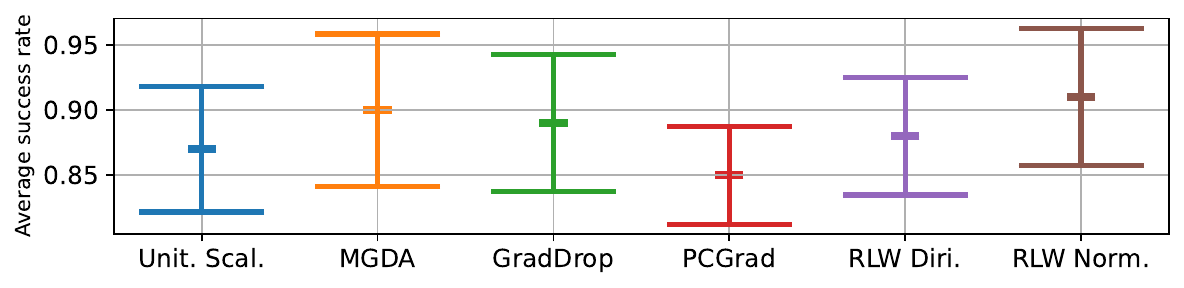}
		\vspace{-18pt}
		\caption{\label{fig:mt10-success} MT10 (10 repetitions).}
	\end{subfigure}\hspace{5pt}
	\begin{subfigure}{0.49\textwidth}
		\centering
		\includegraphics[width=\columnwidth]{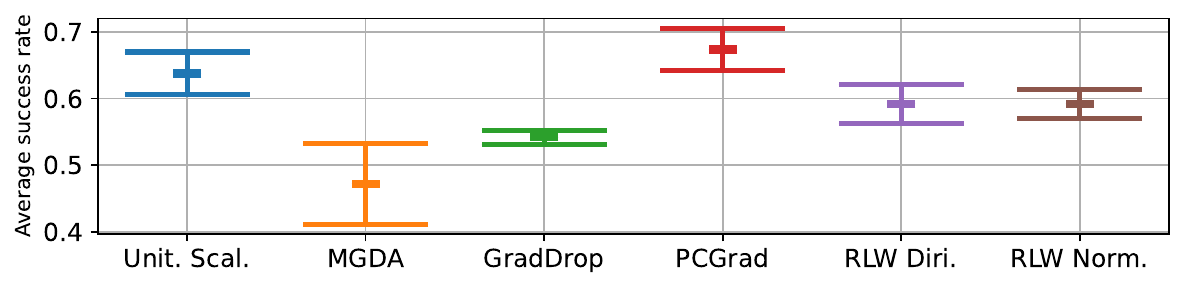}
		\vspace{-18pt}
		\caption{\label{fig:mt50-success} MT50 (10 repetitions).}
	\end{subfigure}
	\\
	\begin{subfigure}{0.49\textwidth}
		\centering
		\includegraphics[width=\textwidth]{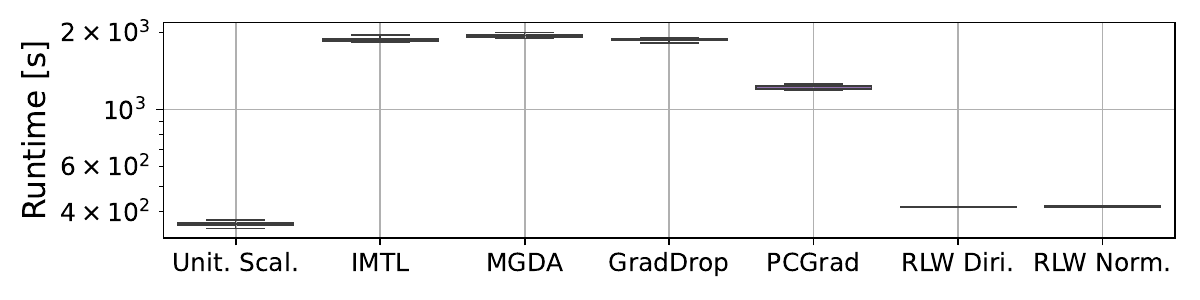}
		\vspace{-18pt}
		\caption{\label{fig:mt10-runtime} MT10 (10 repetitions).}
	\end{subfigure}\hspace{5pt}
	\begin{subfigure}{0.49\textwidth}
		\centering
		\includegraphics[width=\textwidth]{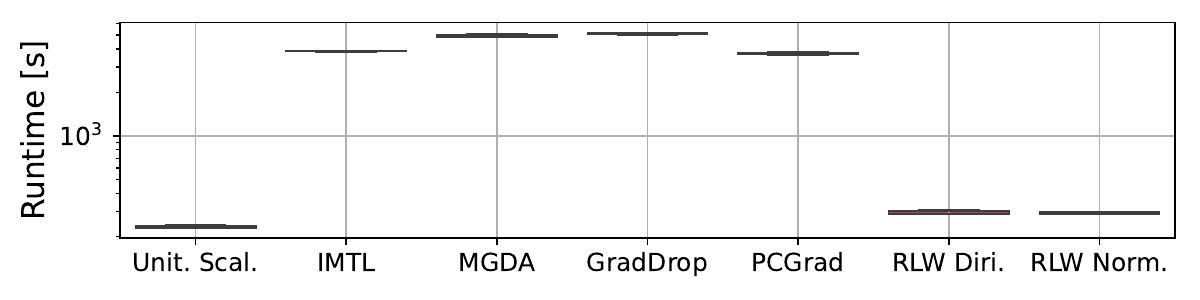}
		\vspace{-18pt}
		\caption{\label{fig:mt50-runtime} MT50 (10 repetitions).}
	\end{subfigure}
	\vspace{-5pt}
	\caption{\label{fig:mt} 
		On Metaworld, none of the \glspl{smto} significantly outperforms Unit. Scal., which is the least expensive method.
		Subfigures \subref{fig:mt10-success}-\subref{fig:mt50-success} report mean and 95$\%$ CI for the best (over the updates) average success rate. Subfigures \subref{fig:mt10-runtime}-\subref{fig:mt50-runtime} show box plots for the training time of 10,000 updates.  }
	\vspace{-10pt}
\end{figure*}

Figure~\ref{fig:mt} reports the best average success rate across the updates and the runtime for 10,000 updates. 
In addition to these summary statistics, reported for consistency with $\S$\ref{sec:sl-experiments}, the learning curves are shown in appendix~\ref{sec:supp-rl-experiments}.
Our MT10 (10 tasks) results in Figure~\ref{fig:mt10-success} show that by stabilizing the baseline using standard~\gls{rl} techniques, unitary scalarization performs on par with other~\glspl{smto}, mirroring our findings in~$\S$\ref{sec:sl-experiments}.
This is in contrast with the previous literature, which reported that PCGrad outperforms unitary scalarization~\citep{Yu2020, Sodhani2021}.
Figure~\ref{fig:mt50-success} presents results on MT50 (50 tasks): similarly to MT10, none of the~\glspl{smto} significantly outperforms unitary scalarization, with PCGrad's average being slightly above unitary scalarization.
We speculate that the stochastic loss rescaling performed by PCGrad (see Proposition \ref{prop:pcgrad}) reduces the differences in task return scales, and expect that methods like PopArt~\citep{van2016learning} would have a similar effect without requiring access to per-task gradients.
While we did not tune hyperparameters for MT50 (we employed those found for MT10), it would be much easier to do that for unitary scalarization due to its lower runtime. In fact, Figure~\ref{fig:mt50-runtime} shows that a single unitary scalarization run takes roughly 15 hours, whereas PCGrad, MGDA and GradDrop require more than a week.
Similarly to MT10, actor regularization pushes the average performance of unitary scalarization higher (see in appendix~\ref{sec:rl-ablations}).
Overall, as in the supervised learning setting, unitary scalarization performs comparably to~\glspl{smto} despite being simpler and less demanding in both memory and compute.
IMTL was unstable on this \gls{rl} benchmark and all of the runs crashed due to numerical overflow. We hence omit IMTL results from the main body of the paper and show its results in Figure~\ref{fig:mt-with-imtl} in appendix~\ref{sec:supp-rl-experiments}, which also describes a possible explanation.
We hypothesize that the instability of IMTL is due to lack of bounds on scaling coefficients. 
See appendix~\ref{ref:rl-hyperparams} for hyperparameter settings and ablation studies.

\section{Regularization in Specialized Multi-Task Optimizers}
\label{sec:critical-analysis}

The empirical results presented in $\S$\ref{sec:experiments} motivate the need to carefully analyze existing \glspl{smto}. We make an initial attempt in this direction by viewing their effects through the lens of regularization.
Let us define a regularizer as a technique to reduce overfitting~\citep{Dietterich1995}.
We first show that the \glspl{smto} considered in~$\S$\ref{sec:experiments} empirically act as regularizers via an ablation study ($\S$\ref{sec:reg-ablation}). We then take a closer look at their behavior, presenting technical results that support their alternative interpretation as regularizers ($\S$\ref{sec:technical-reg}). Finally, $\S$\ref{sec:norm-sum-grad} provides additional empirical backing for some of the technical results.
Unless otherwise stated, we assume that \gls{mtl} methods apply only to $\thetashb$ and that standard gradient-based updates are employed for tasks-specific parameters $\thetaspb$. 
We furthermore adopt the following shorthands: $\mathcal{L}_i (\thetab)$ for $\mathcal{L}_i (f(\thetab, X, i), Y)$, and $\nabla_{\thetab} \mathcal{L}_i$ for $\nabla_{\thetab}\mathcal{L}_i (f(\thetab, X, i), Y)$.

\subsection{Ablation Study} \label{sec:reg-ablation}

\begin{figure}[!b]
	\begin{minipage}{.49\textwidth}
		\centering
		\includegraphics[width=\columnwidth]{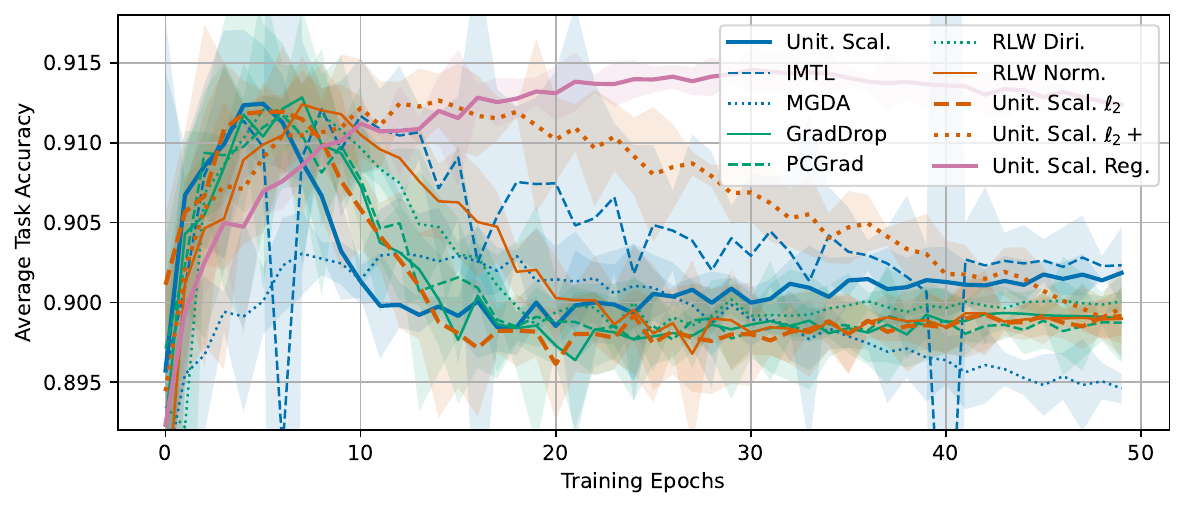}
		\vspace{-8pt} 
		\caption{Mean and 95$\%$ CI (3 runs) avg. task validation accuracy over epochs on CelebA. 
			\glspl{smto} postpone the onset of overfitting, mirroring the effect of $\ell_2$ regularization on unitary~scalarization. 
		}
		\label{fig:celeba-unreg-validation}
		\vspace{-10pt}
	\end{minipage} \hspace{5pt}
	\begin{minipage}{.49\textwidth}
		\centering
		\includegraphics[width=\columnwidth]{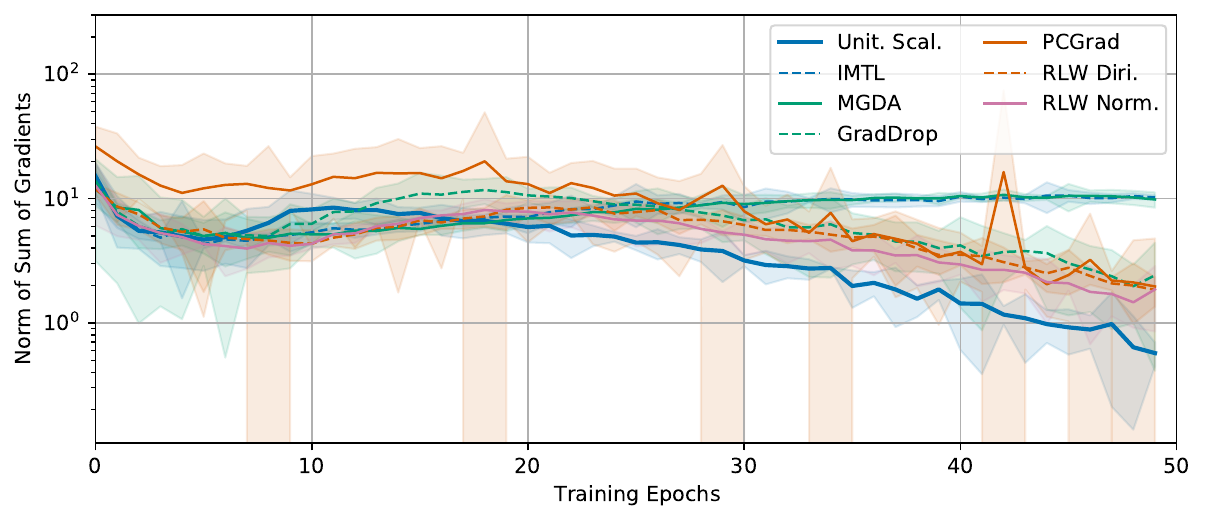}
		\vspace{-8pt} 
		\caption{
			 Mean and $95\%$ CI (3 runs) for $\left\lVert\sum_{i \in \mathcal{T}}  \nabla_{\thetashb} \mathcal{L}_i\right\rVert_2$ on CelebA.
			 \gls{mgda} and \gls{imtl}  converge away from stationary points of unitary scalarization, indicating under-optimization.
			\label{fig:norm-sum-grad}}
		\vspace{-10pt}
	\end{minipage}
\end{figure}

We repeat the experiment from $\S$\ref{sec:celeba-experiments} and remove explicit regularization: no dropout layers are added to the encoder-decoder architecture, and $\lambda=0$ for all optimizers. 
In addition, we examine the behavior of two different $\ell_2$-regularized instances of unitary scalarization:
$\lambda=10^{-4}$ for ``Unit.\ Scal.\ $\ell_2$'', $\lambda=2 \times 10^{-3}$ for ``Unit.\ Scal.\ $\ell_2+$''. 
Figure~\ref{fig:celeba-unreg-validation} shows that \glspl{smto} behave similarly to an $\ell_2$-penalized unitary scalarization. 
Importantly, \glspl{smto} delay overfitting, requiring less early stopping compared to unitary scalarization to obtain comparable performance.
In other words, early stopping is sufficient  for unitary scalarization to perform on par with \glspl{smto}.
Moreover, overfitting is further reduced by ``Unit.\ Scal.\ Reg.'', which plots the regularized unitary scalarization from $\S$\ref{sec:celeba-experiments}, with dropout layers and a weight decay of $\lambda=10^{-3}$.
Finally, Figure~\ref{fig:celeba-unreg-loss} shows that unregularized unitary scalarization and most \glspl{smto} rapidly drive the training loss of each task towards its global optimum. This suggests that the main difficulty of \gls{mtl} is not associated with the optimization of its training objective, but rather to incorporating adequate~regularization.
Additional results are presented in appendix~\ref{sec:unreg}.

\subsection{Technical Results} \label{sec:technical-reg}

\looseness=-1
All the methods considered in $\S$\ref{sec:reg-ablation} regularize more than unitary scalarization. 
While RLW was shown to reduce overfitting by the original authors~\citep[theorem 2]{rlw}, we now provide a collection of novel and existing technical results that potentially explain the regularizing behavior of each of the other algorithms, complementing the presentation from $\S$\ref{sec:setting}.
In particular, we show that \gls{mgda}, \gls{imtl} and PCGrad have a larger convergence set than unitary scalarization, reducing the chances to land on sharp local minima~\citep{Dietterich1995}. Furthermore, GradDrop and PCGrad introduce significant stochasticity, which is often linked to the same effect~\citep{Keskar2017,Kleinberg2018}. We hope these observations will steer further research. 

\paragraph{MGDA}
Let us denote the convex hull of a set $\mathcal{A}$ by $\conv(\mathcal{A})$. We now recall a well-known property of MGDA~\citep{Desideri2012} and relate it to the behavior of unitary scalarization. 
\begin{restatable}{proposition}{mgda}
	\looseness=-1
	The \gls{mgda}~\gls{smto}~\citep{Sener2018} converges to a superset of the convergence points of unitary scalarization. More specifically, it converges to any point $\thetashb^*$ such that: $\mbf{0} \in \conv(\{ \nabla_{\thetashb^*} \mathcal{L}_i \ |\ i \in \mathcal{T}  \})$.
	\label{prop:mgda}
\end{restatable} \vspace{-7pt}
See appendix~\ref{sec:supp-mgda} for a simple proof. As a consequence of Proposition~\ref{prop:mgda}, \gls{mgda} does not necessarily reach a stationary point for $\mathcal{L}^{\text{MT}}$ (that is, a point for which $\sum_{i \in \mathcal{T}} \nabla_{\thetashb} \mathcal{L}_i = \mbf{0} $) or for any of the losses $\mathcal{L}_i$ ($\nabla_{\thetashb} \mathcal{L}_i = \mbf{0} $). For example, any point $\thetashb$ for which two per-task gradients point in opposite directions is Pareto stationary. %
On account of the well-known~\citep{Dietterich1995} relationship between under-optimizing (e.g., early stopping~\citep{Caruana2000,Li2020}) and overfitting, proposition~\ref{prop:mgda} supports the interpretation of \gls{mgda} %
 as a regularizer for equation \cref{eq:unit-scalarization}. 
Empirical evidence that MGDA under-optimizes is provided in $\S$\ref{sec:norm-sum-grad}, Figure~\ref{fig:celeba-unreg-loss}, and Figure~\ref{fig:celeba-unreg-validation}, which shows over-regularization.
Proposition~\ref{prop:mgda} can be trivially extended to the recent Nash-MTL, which shares the same convergence set \citep[Theorem 5.4]{Navon2022}.

\paragraph{IMTL}
We now show that aggregating per-task gradients so that their cosine similarity is the same (equation \cref{eq:imtlg-goal}) yields a constrained steepest-descent algorithm (Proposition \ref{prop:imtl}). This view on the update step of \gls{imtl} leads to a novel analysis of its convergence points (corollary \ref{cor:imtl}). Proofs can be found in appendix~\ref{sec:supp-imtl}. We will denote by $\aff(\mathcal{A})$ the affine hull of a set $\mathcal{A}$.
\begin{restatable}{proposition}{imtl}
	IMTL by \citet{Liu2021} updates $\thetashb$ by taking a step in the \emph{steepest descent direction whose} cosine similarity with per-task gradients is the same across tasks.
	\label{prop:imtl}
\end{restatable}
\begin{restatable}{corollary}{imtlcor}
	\gls{imtl} by \citet{Liu2021} converges to a superset of the Pareto-stationary points for $\thetashb$ (and hence of the convergence points of the unitary scalarization). 
	More specifically, it converges to any point $\thetashb^*$ such that: $\begin{array}{l}	
	\mbf{0} \in \aff\left(\left\{ \nicefrac{\nabla_{\thetashb^*} \mathcal{L}_i}{\left\lVert\nabla_{\thetashb^*} \mathcal{L}_i\right\rVert} \ |\ i \in \mathcal{T}  \right\}\right).\end{array}$
	\label{cor:imtl}
\end{restatable} \vspace{-5pt}
As seen for \gls{mgda}, corollary~\ref{cor:imtl} implies that, even if the employed model~$f$ has the capacity to reach the minimal loss on $\mathcal{L}^{\text{MT}}$, \gls{imtl} may stop before reaching a stationary point. 
Recalling the relationship between under-optimizing and overfitting~\citep{Dietterich1995}, this supports the interpretation of \gls{imtl} as a regularizer for equation \cref{eq:unit-scalarization}.
This is empirically shown in $\S$\ref{sec:norm-sum-grad}, Figures~\ref{fig:celeba-unreg-validation}, \ref{fig:celeba-unreg-loss}. In particular, unitary scalarization reaches the same average performance of \gls{imtl} but requires earlier~stopping. 

\paragraph{PCGrad} 
We provide an alternative characterization of the PCGrad update rule, highlighting its stochasticity in the context of its interpretation as loss rescaling~\citep{Liu2021,rlw}.
See appendix~\ref{sec:supp-pcgrad} for~a~proof. \vspace{-5pt}
\begin{restatable}{proposition}{pcgrad} 
	PCGrad is equivalent to a dynamic, and possibly stochastic, loss rescaling for $\thetashb$. At each iteration, per-task gradients are rescaled as follows:
	$$\begin{array}{l}
		\nabla_{\thetashb} \mathcal{L}_i \leftarrow \left( 1 + \sum_{j \in \mathcal{T} \setminus \{i\}} d_{ji} \right) \nabla_{\thetashb} \mathcal{L}_i, \ d_{ji} \in \smash{\left[0, \frac{\left\lVert\nabla_{\thetashb} \mathcal{L}_j\right\rVert}{\left\lVert\nabla_{\thetashb} \mathcal{L}_i\right\rVert} \right]}.\end{array}$$ 
	Furthermore, if $|\mathcal{T}| > 2$, $d_{ji}$ is a random variable, and the above range contains its~support.
	\label{prop:pcgrad}
\end{restatable}  \vspace{-5pt}
\looseness=-1
The results from proposition \ref{prop:pcgrad} can be easily extended to GradVac~\citep{Wang2021}, which generalizes PCGrad's projection onto the normal vector to arbitrary target cosine similarities between per-task gradients.
When $|\mathcal{T}| > 2$, PCGrad corresponds to a stochastic loss re-weighting. 
As such, PCGrad bears many similarities with Random Loss Weighting (RLW)~\citep{rlw}. RLW proposes to sample scalarization weights from standard probability distributions at each iteration, and proves that this leads the better generalization~\citep[theorem 2]{rlw}. 
Indeed, it is well-known that adding noise to stochastic gradient estimations leads the optimization towards flatter minima, and that such minima may reduce overfitting~\citep{Keskar2017,Kleinberg2018}.
In line with the main technical results by \citet{Yu2020}, we now restrict our focus to two-task problems, which allow for an easy description of PCGrad's convergence points. The result is largely based on~\citep[theorem~1]{Yu2020}: we relax some of the assumptions and provide a proof in appendix~\ref{sec:supp-pcgrad}.
\begin{restatable}{corollary}{pcgradcor}
	If $|\mathcal{T}| = 2$, PCGrad will stop at any point where $\cos(\nabla_{\thetashb} \mathcal{L}_1, \nabla_{\thetashb} \mathcal{L}_2) = -1$.
	Furthermore, if $\mathcal{L}_1$ and $\mathcal{L}_2$ are differentiable, and $\nabla_{\thetashb} \mathcal{L}^{\text{MT}}$ is L-Lipschitz with $L > 0$, PCGrad with step size $t < \frac{1}{L}$ converges to a superset of the convergence points of the unitary scalarization. 
	\label{cor:pcgrad}
\end{restatable} \vspace{-5pt}
Corollary~\ref{cor:pcgrad} implies that, when $|\mathcal{T}| = 2$, PCGrad may under-optimize equation \cref{eq:unit-scalarization} as MGDA and \gls{imtl}. In particular, if $\cos(\nabla_{\thetashb} \mathcal{L}_1, \nabla_{\thetashb} \mathcal{L}_2) = -1$, then $\mbf{0} \in \conv(\{ \nabla_{\thetashb} \mathcal{L}_1, \nabla_{\thetashb} \mathcal{L}_2\})$ (see proposition \ref{prop:mgda}).
We believe that PCGrad's stochasticity and enlarged convergence set potentially explain its regularizing effect. 

\paragraph{GradDrop}
\looseness=-1
While the motivation behind GradDrop is to avoid entry-wise gradient conflicts across tasks, the main property of the method is to drive the optimization towards ``joint minima": points that are stationary for all the individual tasks at once~\citep[proposition 1]{Chen2020}. 
In other words: $\nabla_{\thetashb} \mathcal{L}_i = \mbf{0} \ \forall \ i \in \mathcal{T}$.
While this property is desirable, we show that it holds beyond GradDrop, and independently of the gradient directions. Under strong assumptions 
on the model capacity, the above property would trivially hold for unitary scalarization (proposition \ref{prop:us-graddrop}, appendix~\ref{sec:supp-graddrop}).
Proposition \ref{prop:random-graddrop} shows that it holds for a simple randomized version of unitary~scalarization, which we name Random Grad Drop~(RGD). 
\begin{restatable}{proposition}{randomgraddrop}
	Let $\mathcal{L}^{\text{RGD}}(\thetashb) := \sum_{i \in \mathcal{T}} u_i \mathcal{L}_i(\thetashb)$, where $u_i \sim \text{Bernoulli}(p) \ \forall i \in \mathcal{T}$ and $p \in (0, 1]$. The gradient $\nabla_{\thetashb} \mathcal{L}^{\text{RGD}}$ is always zero if and only if $\nabla_{\thetashb} \mathcal{L}_i = \mbf{0} \ \forall i \in \mathcal{T}$. In other words, the result from \citep[proposition 1]{Chen2020} can be obtained without any information on the sign of per-task~gradients.
	\label{prop:random-graddrop}
\end{restatable}  \vspace{-5pt}
Proposition~\ref{prop:random-graddrop} (see appendix~\ref{sec:supp-graddrop} for a simple proof) shows that an inexpensive sign-independent stochastic scalarization shares GradDrop's main reported property.
$\mathcal{L}^{\text{RGD}}$ can be directly cast an instance of RLW, and hence as a regularization method~\citep{Keskar2017,Kleinberg2018}.
Furthermore, Figure~\ref{fig:celeba-graddrop} in appendix \ref{sec:random-graddrop} shows that the empirical results of GradDrop on CelebA~\citep{liu2015faceattributes} are closely matched by a sign-agnostic gradient masking, partly undermining the conflicting gradients assumption.
We believe that the above results, along with the authors' original experiments showing that GradDrop delays overfitting on CelebA~\citep[figure 3]{Chen2020}, suggest that GradDrop behaves as a regularizer.

\subsection{Under-Optimization: Empirical Study} \label{sec:norm-sum-grad}

As seen in $\S$\ref{sec:technical-reg}, MGDA and IMTL might under-optimize equation~\cref{eq:unit-scalarization} compared to unitary scalarization due to their larger convergence sets. 
In order to assess whether this is empirically the case, we estimate $\left\lVert\sum_{i \in \mathcal{T}}  \nabla_{\thetashb} \mathcal{L}_i\right\rVert_2$, the norm of the unitary scalarization update on shared parameters~$\thetashb$, for all optimizers throughout the unregularized CelebA experiment from $\S$\ref{sec:reg-ablation}. 
Large magnitudes for $\left\lVert\sum_{i \in \mathcal{T}}  \nabla_{\thetashb} \mathcal{L}_i\right\rVert_2$ towards convergence would indicate that \glspl{smto} steer optimization far from stationary points of unitary scalarization, resulting in under-optimization.
We compute the update norm on the mini-batch loss every $100$ updates, and report the per-epoch average in Figure~\ref{fig:norm-sum-grad}.
Compared with unitary scalarization, most \glspl{smto} have smaller or comparable update magnitude in the first $15$ epochs.
However, towards convergence, \glspl{smto} display larger $\left\lVert\sum_{i \in \mathcal{T}}  \nabla_{\thetashb} \mathcal{L}_i\right\rVert_2$ compared to unitary scalarization. 
In particular, \gls{imtl} and \gls{mgda} have the largest norm, denoting significant empirical under-optimization. The additional stochasticity of RLW, PCGrad, and GradDrop also appears to lead to larger norm values than unitary scalarization, yet to a lesser degree.
Given that \gls{mgda} and \gls{imtl} incur a larger loss than unitary scalarization in later epochs (see Figure~\ref{fig:celeba-unreg-loss} in appendix~\ref{sec:unreg}), we can conclude that they guide optimization towards regions of the parameter space that under-optimize equation~\cref{eq:unit-scalarization}, providing empirical support for our analysis.

\section{Conclusions} \label{sec:conclusion}

This paper made two main contributions.  
First, we evaluated popular~\glspl{smto} using a single experimental pipeline, including previously unpublished results of MGDA, IMTL, RLW, and GradDrop in the~\gls{rl} setting.
Surprisingly, our evaluation showed that none of the~\glspl{smto} consistently outperform unitary scalarization, the simplest and least expensive method.
Second, in order to explain our surprising results, we postulate that~\glspl{smto} act as regularizers and present an analysis that supports our hypothesis.
We believe our work calls for further reevaluation of progress in developing principled and efficient \gls{mtl}~algorithms.

We conclude by addressing the limitations of our work.
While we covered a wide range of popular benchmarks, we do not exclude the existence of settings where unitary scalarization underperforms: discovering them is an interesting direction for future work.
Furthermore, our experimental results were obtained via grid searches under limited compute resources: some of the methods might benefit from further fine-tuning. Nevertheless, we remark that fine-tuning will be easier for unitary scalarization due to its shorter runtimes.
Finally, we presented the regularization hypothesis only as a partial explanation of our results: we hope it will steer further analysis and consequently improve the understanding of \gls{mtl}.

\section*{Acknowledgements}

VK was funded by Samsung R\&D Institute UK through the EPSRC Centre for Doctoral Training (CDT) in Autonomous Intelligent Machines and Systems (AIMS) at the University of Oxford .
ADP was funded by EPSRC for the AIMS CDT, grant EP/L015987/1, and by an IBM PhD fellowship. 
SW has received funding from the European Research Council under the European Union’s Horizon 2020 research and innovation programme (grant agreement number 637713).
The experiments were made possible by a generous equipment grant from NVIDIA.
We would like to thank~\citet{Sodhani2021,rlw} and~\citet{Sener2018} for publicly releasing their code. %
The authors thank Kristian Hartikainen for helpful comments on the~\gls{rl} experiments, and Gabriel Gama for spotting a bug in the logging of training statistics for supervised learning.
VK thanks Ryota Tomioka for useful discussions on multitask optimization.

\bibliography{mtlopt}
\bibliographystyle{abbrvnat}
\section*{Checklist}

\begin{enumerate}

\item For all authors...
\begin{enumerate}
  \item Do the main claims made in the abstract and introduction accurately reflect the paper's contributions and scope?
    \answerYes{}
  \item Did you describe the limitations of your work?
    \answerYes{} see $\S$\ref{sec:conclusion}.
  \item Did you discuss any potential negative societal impacts of your work?
    \answerYes{} due to space constraints, we provide a discussion in appendix~\ref{sec:impact}.
  \item Have you read the ethics review guidelines and ensured that your paper conforms to them?
    \answerYes{}
\end{enumerate}

\item If you are including theoretical results...
\begin{enumerate}
  \item Did you state the full set of assumptions of all theoretical results?
    \answerYes{}
        \item Did you include complete proofs of all theoretical results?
    \answerYes{} we provide full proofs in the Appendix, and refer to them in the main body of the paper.
\end{enumerate}

\item If you ran experiments...
\begin{enumerate}
  \item Did you include the code, data, and instructions needed to reproduce the main experimental results (either in the supplemental material or as a URL)?
    \answerYes{} we provide the code and the instructions in the supplemental material.
  \item Did you specify all the training details (e.g., data splits, hyperparameters, how they were chosen)?
    \answerYes{}
        \item Did you report error bars (e.g., with respect to the random seed after running experiments multiple times)?
    \answerYes{}
        \item Did you include the total amount of compute and the type of resources used (e.g., type of GPUs, internal cluster, or cloud provider)?
    \answerYes{} see appendix~\ref{sec:sl-setup}.
\end{enumerate}

\item If you are using existing assets (e.g., code, data, models) or curating/releasing new assets...
\begin{enumerate}
  \item If your work uses existing assets, did you cite the creators?
    \answerYes{}
  \item Did you mention the license of the assets?
    \answerYes{} appendix~\ref{sec:licenses} describes licenses of all benchmarks and implementations we used for our work.
  \item Did you include any new assets either in the supplemental material or as a URL?
    \answerYes{} we include the code and the instructions on how to replicate the experiments into the supplemental material.
  \item Did you discuss whether and how consent was obtained from people whose data you're using/curating?
    \answerNA{}
  \item Did you discuss whether the data you are using/curating contains personally identifiable information or offensive content?
    \answerNA{}
\end{enumerate}

\item If you used crowdsourcing or conducted research with human subjects...
\begin{enumerate}
  \item Did you include the full text of instructions given to participants and screenshots, if applicable?
    \answerNA{}
  \item Did you describe any potential participant risks, with links to Institutional Review Board (IRB) approvals, if applicable?
    \answerNA{}
  \item Did you include the estimated hourly wage paid to participants and the total amount spent on participant compensation?
    \answerNA{}
\end{enumerate}

\end{enumerate}

\iftoggle{appendix}{
\clearpage
\appendix
\section{Societal Impact}
\label{sec:impact}

Due to the object of its study, our work does not have a direct societal impact.
However, as any machine learning paper, it can potentially negatively effect the society through automation and loss of jobs.
While it is hard to anticipate any particular risk, as any technology, if not regulated properly, it might lead to growing social and economic inequality.

On the positive side, our work might have a positive environmental impact since it advocates for simpler and more economical methods which will reduce energy consumption in data centers. 
Finally, simpler methods are usually easier to understand, which is beneficial in terms of explainability, an important factor for real-life applications.

\section{Supplement to the Overview of Multi-Task Optimizers} \label{sec:supp-overview}

This section presents the proofs and the technical results omitted from section~\ref{sec:critical-analysis}, along with a description of the use of per-task gradients with respect to the last shared activation for encoder-decoder architectures (usually less expensive than per-task gradients with respect to shared parameters).

\subsection{MGDA} \label{sec:supp-mgda}
\mgda*
\begin{proof}
	As shown by \citet{Desideri2012}, equation \cref{eq:dual-mgda} is a simplex-constrained norm-minimization problem. In other words, the argument of the minimum is the projection of $\mbf{0}$ onto the feasible set. Therefore: $$\gb = \mbf{0} \iff \mbf{0} \in \conv(\{ \nabla_{\thetashb} \mathcal{L}_i \ |\ i \in \mathcal{T}  \}).$$
	It then suffices to point out that $\sum_{i \in \mathcal{T}} \nabla_{\thetashb} \mathcal{L}_i = \mbf{0} \iff \sum_{i \in \mathcal{T}} \frac{1}{|\mathcal{T}|} \nabla_{\thetashb} \mathcal{L}_i = \mbf{0} \Rightarrow \mbf{0} \in \conv(\{ \nabla_{\thetashb} \mathcal{L}_i \ |\ i \in \mathcal{T}  \})$ to conclude the proof.
\end{proof}

Due to the cost of computing per-task gradients, \citet{Sener2018} propose MGDA-UB, which replaces the gradients wrt the parameters $\nabla_{\thetashb} \mathcal{L}_i$ with the gradients wrt the shared activation $\nabla_{\zb} \mathcal{L}_i$ in the computation of the coefficients of $\gb = -\sum_i \alpha_i \nabla_{\thetashb} \mathcal{L}_i$. 
This yields an upper bound on the objective of equation \cref{eq:dual-mgda}, thus restricting the set of points the algorithm convergences to.
Rather than directly relying on $\nabla_{\thetashb} \mathcal{L}_i$, $\gb$ can then be obtained by computing the gradient of $\sum_{i \in \mathcal{T}} \alpha_i \mathcal{L}_i$ via reverse-mode differentiation, hence saving memory and compute. 

\begin{corollary}
	The MGDA-UB~\gls{smto} by \citet{Sener2018} converges to any point such that: $\mbf{0} \in \conv(\{ \nabla_{\zb} \mathcal{L}_i \ |\ i \in \mathcal{T}  \})$.
	Furthermore, if $\frac{\partial \zb}{\partial \thetashb}$ is non-singular, it converges to a superset of the convergence points of the unitary scalarization. 
	\label{cor:mgda-ub}
\end{corollary}
\begin{proof}
	The first part of the proof proceeds as the proof of proposition \ref{prop:mgda}, noting that the MGDA-UB update is associated to the following problem:
	\begin{equation*}
		\begin{aligned}
			\max_{\alphab} \qquad & -\frac{1}{2} \left\lVert \gb \right \rVert_2^2 \\
			\text{s.t. } \qquad & \sum_i \alpha_i \nabla_{\zb} \mathcal{L}_i  = -\gb, \quad \sum_{i \in \mathcal{T}} \alpha_i =  1, \\
			& \alpha_i \geq 0 \qquad \forall \ i \in \mathcal{T}.
		\end{aligned}	
	\end{equation*}
	In order to show that a stationary point of the unitary scalarization satisfies $\mbf{0} \in \conv(\{ \nabla_{\zb^*} \mathcal{L}_i \ |\ i \in \mathcal{T}  \})$, we will assume $\frac{\partial \zb}{\partial \thetashb}$ is non-singular, as done by \citet[theorem 1]{Sener2018}. Then, relying on the chain rule, the result follows from: 
	\begin{equation*}
		\begin{aligned}
			\sum_{i \in \mathcal{T}} \nabla_{\thetashb} \mathcal{L}_i = \mbf{0} &\iff \sum_{i \in \mathcal{T}} \frac{1}{|\mathcal{T}|} \nabla_{\thetashb} \mathcal{L}_i = \mbf{0} \\
			& \iff \sum_{i \in \mathcal{T}} \frac{\frac{\partial \zb}{\partial \thetashb}}{|\mathcal{T}|} \nabla_{\zb} \mathcal{L}_i = \mbf{0} \\
			& \iff \left(\frac{\partial \zb}{\partial \thetashb}\right)^{-1} \frac{\partial \zb}{\partial \thetashb} \sum_{i \in \mathcal{T}} \frac{1}{|\mathcal{T}|} \nabla_{\zb} \mathcal{L}_i = \mbf{0} \\
			& \iff \sum_{i \in \mathcal{T}} \frac{1}{|\mathcal{T}|} \nabla_{\zb} \mathcal{L}_i = \mbf{0} \\
			&\Rightarrow \mbf{0} \in \conv(\{ \nabla_{\zb} \mathcal{L}_i \ |\ i \in \mathcal{T}  \})
		\end{aligned}	
	\end{equation*}
\end{proof}

\subsection{IMTL} \label{sec:supp-imtl}
\imtl*
\begin{proof}
	First, equation \cref{eq:imtlg-goal} solves the linear system in $\alphab := [\alpha_1, \dots, \alpha_m]$ given by:
	\begin{equation*}
	\begin{aligned}
	& \gb^T \left( \frac{ \nabla_{\thetashb} \mathcal{L}_1}{\left\lVert\nabla_{\thetashb} \mathcal{L}_1\right\rVert} - \frac{\nabla_{\thetashb} \mathcal{L}_i}{\left\lVert\nabla_{\thetashb} \mathcal{L}_i\right\rVert}\right) = \mbf{0} \qquad  \forall \ i \in \mathcal{T} \setminus \{1\}, \\[3pt]
	& \gb = -\sum_i \alpha_i \nabla_{\thetashb} \mathcal{L}_i, \quad \sum_{i \in \mathcal{T}} \alpha_i = 1,
	\end{aligned}
	\end{equation*}
	which corresponds to finding a point of $\mathcal{A}' := \aff(\left\{ \nabla_{\thetashb} \mathcal{L}_i |\ i \in \mathcal{T} \right\})$ which is orthogonal to $\mathcal{A} := \aff\left(\left\{ \frac{\nabla_{\thetashb} \mathcal{L}_i}{\left\lVert\nabla_{\thetashb} \mathcal{L}_i \right\rVert} |\ i \in \mathcal{T} \right\}\right)$. 
	To see this, it suffices to point out that any point orthogonal to $\mathcal{A}$ is also orthogonal to the vector subspace spanned by differences of vectors belonging to $\mathcal{A}$. As this subspace  has $m-1$ dimensions, any vector orthogonal to $\left(\frac{\nabla_{\thetashb} \mathcal{L}_1}{\left\lVert\nabla_{\thetashb} \mathcal{L}_1\right\rVert} -  \frac{\nabla_{\thetashb} \mathcal{L}_i}{\left\lVert\nabla_{\thetashb} \mathcal{L}_i\right\rVert} \right)$ for each $i \in \mathcal{T} \setminus \{1\}$ is orthogonal to the entire subspace.
	
	Second, consider the problem of finding a point in $\mathcal{A}$ that is orthogonal to the linear subspace spanned by differences of vectors in $\mathcal{A}$. In other words, we seek the projection of $\mbf{0}$ onto $\mathcal{A}$. Recalling the definition of $\mathcal{A}$, we can write:
	\begin{equation}
	\begin{aligned}
	\max_{\alphab} \qquad & -\frac{1}{2} \left\lVert \gb' \right \rVert_2^2 \\
	\text{s.t. } \qquad & \sum_i \alpha_i \frac{\nabla_{\thetashb} \mathcal{L}_i}{\left\lVert\nabla_{\thetashb} \mathcal{L}_i\right\rVert}   = -\gb', \quad \sum_i \alpha_i =  1.
	\label{eq:dual-imtlg}
	\end{aligned}	
	\end{equation}
	The solution of equation \cref{eq:dual-imtlg} is always collinear to the solution of equation \cref{eq:imtlg-goal}. In fact, if a vector $\gb \in \mathcal{A}'$ is orthogonal to the affine subspace $\mathcal{A}$ (or to the linear subspace spanned by differences of its members), then $\gamma \gb =  \left( - \gamma \sum_i \left(\alpha_i \left\lVert\nabla_{\thetashb} \mathcal{L}_i\right\rVert \right) \frac{\nabla_{\thetashb} \mathcal{L}_i}{\left\lVert\nabla_{\thetashb} \mathcal{L}_i\right\rVert} \right)$ is orthogonal to $\mathcal{A}$ as well, and $\gamma = \frac{1}{\sum_i \left(\alpha_i \left\lVert\nabla_{\thetashb} \mathcal{L}_i\right\rVert \right)} \implies \gamma \gb \in \mathcal{A}$.
	
	Finally, equation \cref{eq:dual-imtlg} differs from equation \cref{eq:dual-mgda} in two aspects: $\alphab$ is not constrained to be non-negative (hence the convex hull is replaced by the affine hull), and the task vectors are normalized. Therefore, equation \cref{eq:dual-imtlg} is the dual of:
	\begin{equation}
	\begin{aligned}
	\min_{\gb, \epsilon} \qquad & \epsilon + \frac{1}{2} \left\lVert \gb \right \rVert_2^2 \\
	\text{s.t. } \qquad & \frac{\nabla_{\thetashb} \mathcal{L}_i^T}{\left\lVert\nabla_{\thetashb} \mathcal{L}_i\right\rVert} \gb = \epsilon \qquad \forall \ i \in \left\{1, \dots,  m \right\}.
	\label{eq:primal-imtlg}
	\end{aligned}	
	\end{equation}
	The proposition then follows by comparing equation \cref{eq:primal-imtlg} with  equation \cref{eq:primal-mgda}, and recalling that IMTL-L only adds a scaling factor to the chosen update direction.
\end{proof}

\imtlcor*
\begin{proof}
	Inspecting equation \cref{eq:primal-imtlg}, which yields a collinear point to the \gls{imtl} update, reveals that \gls{imtl} might converge to non Pareto-stationary points: due to the restrictive equality constraints, the minimizer of equation \cref{eq:primal-imtlg} might be $\mbf{0}$ even if a descent direction exists. 
	Furthermore, its dual, equation \cref{eq:dual-imtlg}, implies that: 
	\begin{equation*}
	\begin{aligned}
	\gb = \mbf{0} &\iff \mbf{0} \in \aff\left(\left\{ \frac{\nabla_{\thetashb} \mathcal{L}_i}{\left\lVert\nabla_{\thetashb} \mathcal{L}_i\right\rVert} \ |\ i \in \mathcal{T}  \right\}\right)\\
	&\iff \mbf{0} \in \aff\left(\left\{ \nabla_{\thetashb} \mathcal{L}_i\ |\ i \in \mathcal{T}  \right\}\right),
	\end{aligned}
	\end{equation*}
	which, noting that $\conv(\mathcal{A}) \subseteq \aff(\mathcal{A})$ for any $\mathcal{A}$, concludes the proof.
\end{proof}

Similarly to MGDA-UB, \citet{Liu2021} advocate using $\nabla_{\zb} \mathcal{L}_i$ in place of $\nabla_{\thetashb} \mathcal{L}_i$ while solving equation \cref{eq:imtlg-goal}, typically reducing the cost of computing the coefficients of $\gb = -\sum_i \alpha_i \nabla_{\thetashb} \mathcal{L}_i$. 

\begin{corollary}
	When employing the approximation of problem \cref{eq:imtlg-goal} that relies on $\nabla_{\zb} \mathcal{L}_i$, IMTL by \citet{Liu2021} converges to $\mbf{0} \in \aff\left(\left\{ \frac{\nabla_{\zb} \mathcal{L}_i}{\left\lVert\nabla_{\zb} \mathcal{L}_i\right\rVert} \ |\ i \in \mathcal{T}  \right\}\right)$.
	If $\frac{\partial \zb}{\partial \thetashb}$ is non-singular, this is a superset of of the convergence points of the unitary scalarization.
	\label{cor:imtl-ub}
\end{corollary}
\begin{proof}
	Following the proof of proposition \ref{prop:imtl}, the following problem yields a collinear point to the $\nabla_{\zb} \mathcal{L}_i$-approximate IMTL update:
	\begin{equation*}
		\begin{aligned}
			\max_{\alphab} \qquad & -\frac{1}{2} \left\lVert \gb' \right \rVert_2^2 \\
			\text{s.t. } \qquad & \sum_i \alpha_i \frac{\nabla_{\zb} \mathcal{L}_i}{\left\lVert\nabla_{\zb} \mathcal{L}_i\right\rVert}   = -\gb', \quad \sum_i \alpha_i =  1.
		\end{aligned}	
	\end{equation*}
	Therefore: 
	\begin{equation*}
		\gb = \mbf{0} \iff \mbf{0} \in \aff\left(\left\{ \frac{\nabla_{\zb} \mathcal{L}_i}{\left\lVert\nabla_{\zb} \mathcal{L}_i\right\rVert} \ |\ i \in \mathcal{T}  \right\}\right).
	\end{equation*}
	Finally, assuming $\frac{\partial \zb}{\partial \thetashb}$ is non-singular, we can replicate the procedure in the proof of corollary \ref{cor:mgda-ub} to get: 
	\begin{equation*}
		\begin{aligned}
			&\sum_{i \in \mathcal{T}} \nabla_{\thetashb} \mathcal{L}_i = \mbf{0} \iff \sum_{i \in \mathcal{T}} \frac{1}{|\mathcal{T}|} \nabla_{\zb} \mathcal{L}_i = \mbf{0} \\
			& \iff \sum_{i \in \mathcal{T}} \frac{\left\lVert\nabla_{\zb} \mathcal{L}_i\right\rVert}{|\mathcal{T}|} \frac{\nabla_{\zb} \mathcal{L}_i}{\left\lVert\nabla_{\zb} \mathcal{L}_i\right\rVert} = \mbf{0} \\
			& \iff \left( \frac{|\mathcal{T}|}{\sum_{i \in \mathcal{T}} \left( \left\lVert\nabla_{\zb} \mathcal{L}_i\right\rVert \right)} \right) \sum_{i \in \mathcal{T}} \frac{\left\lVert\nabla_{\zb} \mathcal{L}_i\right\rVert}{|\mathcal{T}|} \frac{\nabla_{\zb} \mathcal{L}_i}{\left\lVert\nabla_{\zb} \mathcal{L}_i\right\rVert} = \mbf{0} \\
			&\Rightarrow \mbf{0} \in \conv\left(\left\{ \frac{\nabla_{\zb} \mathcal{L}_i}{\left\lVert\nabla_{\zb} \mathcal{L}_i\right\rVert} \ |\ i \in \mathcal{T}  \right\}\right)\\
			&\Rightarrow \mbf{0} \in \aff\left(\left\{ \frac{\nabla_{\zb} \mathcal{L}_i}{\left\lVert\nabla_{\zb} \mathcal{L}_i\right\rVert} \ |\ i \in \mathcal{T}  \right\}\right),
		\end{aligned}	
	\end{equation*}
	which shows that $\aff\left(\left\{ \frac{\nabla_{\zb} \mathcal{L}_i}{\left\lVert\nabla_{\zb} \mathcal{L}_i\right\rVert} \ |\ i \in \mathcal{T}  \right\}\right)$ contains the convergence points of the unitary scalarization.
\end{proof}

\subsection{PCGrad} \label{sec:supp-pcgrad}

\pcgrad*
\begin{proof}
	We start by pointing out that:
	\begin{equation*}
	\begin{aligned}
	\left[\frac{-\gb_i^T \nabla_{\thetashb} \mathcal{L}_j (\xb)}{\left\lVert \nabla_{\thetashb} \mathcal{L}_j \right\rVert^2}\right]_+ 
	&= \left[\frac{-\gb_i^T \nabla_{\thetashb} \mathcal{L}_j (\xb)}{\left\lVert \nabla_{\thetashb} \mathcal{L}_j \right\rVert}\right]_+ \frac{1}{{\left\lVert \nabla_{\thetashb} \mathcal{L}_j \right\rVert}} \\
	&=\left[ - \cos(\gb_i, \nabla_{\thetashb} \mathcal{L}_j) \left\lVert \gb_i \right\rVert \right]_+ \frac{1}{{\left\lVert \nabla_{\thetashb} \mathcal{L}_j \right\rVert}} \\
	& \in \left[0, \frac{\left\lVert \gb_i \right\rVert}{\left\lVert \nabla_{\thetashb} \mathcal{L}_j \right\rVert}\right].
	\end{aligned}
	\end{equation*}
	As $\gb_i$ is obtained by iterative projections of $\nabla_{\thetashb} \mathcal{L}_i$ onto the normals of $\nabla_{\thetashb} \mathcal{L}_j\ \forall j \in \mathcal{T} \setminus \{i\}$, and the norm of a vector can only decrease or remain unvaried after projections, we can write the coefficient of each $\gb_i$ update as:
	\begin{equation*}
	d_{ij} := \left[\frac{-\gb_i^T \nabla_{\thetashb} \mathcal{L}_j (\xb)}{\left\lVert \nabla_{\thetashb} \mathcal{L}_j \right\rVert^2}\right]_+ \in \left[0, \frac{\left\lVert \nabla_{\thetashb} \mathcal{L}_i \right\rVert}{\left\lVert \nabla_{\thetashb} \mathcal{L}_j \right\rVert}\right], \enskip \forall i \neq j.
	\end{equation*}
	Furthermore, if $|\mathcal{T}| > 2$ the contraction factor $\frac{\left\lVert \gb_i \right\rVert}{\left\lVert \nabla_{\thetashb} \mathcal{L}_i \right\rVert}$ for the norm of $g_i$ depends on the ordering of the projections, which is stochastic by design~\citep{Yu2020}. Therefore, $d_{ij}$ a random variable whose support is contained in $\left[0, \frac{\left\lVert \nabla_{\thetashb} \mathcal{L}_i \right\rVert}{\left\lVert \nabla_{\thetashb} \mathcal{L}_j \right\rVert}\right]$.
	Finally, exploiting the definition of $d_{ij}$,  we can re-write equation \cref{eq:pcgrad-update} as:
	\begin{equation*}
	\begin{aligned}
	-\gb 
	&= \sum_{i \in \mathcal{T}} \nabla_{\thetashb} \mathcal{L}_i + \sum_{i \in \mathcal{T}} \sum_{j \in \mathcal{T} \setminus \{i\}} d_{ij} \nabla_{\thetashb} \mathcal{L}_j  = \sum_{i \in \mathcal{T}} \nabla_{\thetashb} \mathcal{L}_i + \sum_{j \in \mathcal{T}} \sum_{i \in \mathcal{T} \setminus \{j\}} d_{ji} \nabla_{\thetashb} \mathcal{L}_i \\
	&= \sum_{j \in \mathcal{T}} \nabla_{\thetashb} \mathcal{L}_j + \sum_{j \in \mathcal{T}} \sum_{i \in \mathcal{T} \setminus \{j\}} d_{ji} \nabla_{\thetashb} \mathcal{L}_i = \sum_{j \in \mathcal{T}}  \left(  \sum_{i \in \mathcal{T} \setminus \{j\}} d_{ji} \nabla_{\thetashb} \mathcal{L}_i + \nabla_{\thetashb} \mathcal{L}_j  \right).
	\end{aligned}
	\end{equation*}
	Introducing (and then removing, using their definition) dummy variables $d_{jj}=1$:
	\begin{equation*}
	\begin{aligned}
	-\gb 
	&= \sum_{j \in \mathcal{T}}  \left(  \sum_{i \in \mathcal{T} \setminus \{j\}} d_{ji} \nabla_{\thetashb} \mathcal{L}_i + d_{jj} \nabla_{\thetashb} \mathcal{L}_j  \right) = \sum_{j \in \mathcal{T}}  \left(  \sum_{i \in \mathcal{T}} d_{ji} \nabla_{\thetashb} \mathcal{L}_i \right) = \sum_{i \in \mathcal{T}}  \left(  \sum_{j \in \mathcal{T}} d_{ji} \nabla_{\thetashb} \mathcal{L}_i \right) \\
	&=  \sum_{i \in \mathcal{T}} \nabla_{\thetashb} \mathcal{L}_i  \left(  \sum_{j \in \mathcal{T}} d_{ji} \right) = \sum_{i \in \mathcal{T}} \nabla_{\thetashb} \mathcal{L}_i  \left( 1 + \sum_{j \in \mathcal{T} \setminus \{i\}} d_{ji} \right),
	\end{aligned}
	\end{equation*}
	
	from which the result trivially follows.
\end{proof}

\pcgradcor*
\begin{proof}
	Let us start from the first statement, which does not require any assumption on the loss landscape. From proposition \ref{prop:pcgrad}, we get:
	\begin{equation*} 
	\hspace{-10pt}
	\begin{aligned}
	-\gb 
	&= \nabla_{\thetashb} \mathcal{L}_1  \left( 1 + d_{21} \right) + \nabla_{\thetashb} \mathcal{L}_2  \left( 1 + d_{12} \right)\\
	&=  \left( 1 + \left[\frac{- \cos(\nabla_{\thetashb} \mathcal{L}_1, \nabla_{\thetashb} \mathcal{L}_2) \left\lVert \nabla_{\thetashb} \mathcal{L}_2 \right\rVert}{\left\lVert \nabla_{\thetashb} \mathcal{L}_1 \right\rVert}\right]_+ \right) \nabla_{\thetashb} \mathcal{L}_1  \\
	& +\left( 1 + \left[\frac{- \cos(\nabla_{\thetashb} \mathcal{L}_1, \nabla_{\thetashb} \mathcal{L}_2) \left\lVert \nabla_{\thetashb} \mathcal{L}_1 \right\rVert}{\left\lVert \nabla_{\thetashb} \mathcal{L}_2 \right\rVert}\right]_+ \right) \nabla_{\thetashb} \mathcal{L}_2,
	\end{aligned}
	\end{equation*}	
	which shows that, in case of conflicting gradient directions, gradient norms are rebalanced proportionally to the angle between them.
	For $\cos(\nabla_{\thetashb} \mathcal{L}_1, \nabla_{\thetashb} \mathcal{L}_2) = -1$, the above evaluates to:
	\begin{equation*}
	-\gb 
	=  \left( \frac{\left\lVert \nabla_{\thetashb} \mathcal{L}_1 \right\rVert+\left\lVert \nabla_{\thetashb} \mathcal{L}_2 \right\rVert}{\left\lVert \nabla_{\thetashb} \mathcal{L}_1 \right\rVert} \right) \nabla_{\thetashb} \mathcal{L}_1  
	+\left( \frac{\left\lVert \nabla_{\thetashb} \mathcal{L}_1 \right\rVert+\left\lVert \nabla_{\thetashb} \mathcal{L}_2 \right\rVert}{\left\lVert \nabla_{\thetashb} \mathcal{L}_2 \right\rVert} \right) \nabla_{\thetashb} \mathcal{L}_2.
	\end{equation*}	
	The first part of the result then follows by pointing out that, if $\cos(\nabla_{\thetashb} \mathcal{L}_1, \nabla_{\thetashb} \mathcal{L}_2) = -1$, then $\nabla_{\thetashb} \mathcal{L}_1 = -\nabla_{\thetashb} \mathcal{L}_2$, and hence $\gb = \mbf{0}$. 
	We remark that a similar proof appears in \citep[theorem 1 and proposition 1]{Yu2020}. However, our derivation relaxes the author's assumptions on $\mathcal{L}^{\text{MT}}$ and is therefore applicable to the training of neural networks.
	
	Finally, given the assumptions on differentiability and smoothness, we need to prove that PCGrad converges to the stationary points of the unitary scalarization: this directly follows from~\citep[proposition 1]{Yu2020}.
\end{proof}

\subsection{GradDrop} \label{sec:supp-graddrop}

\begin{proposition}
	Let us assume, as often demonstrated in the single-task case~\citep{Ma2018,Allen-Zhou2019}, that the multi-task network has the capacity to interpolate the data on all tasks at once: $\min_{\thetab} \mathcal{L}^{\text{MT}}= \sum_{i \in \mathcal{T}} \min_{\thetab} \mathcal{L}_i$, and that its training by gradient descent attains such global minimum.  Then, if $\inf_{\thetab} \mathcal{L}_i > -\infty\ \forall \ i \in \mathcal{T}$, unitary scalarization converges to a joint~minimum.
	\label{prop:us-graddrop}
\end{proposition}
\begin{proof}
	It suffices to point out that if $\mathcal{L}^{\text{MT}}(\thetab^*)=\sum_{i \in \mathcal{T}} \min_{\thetab} \mathcal{L}_i$, then the globally optimal loss is attained for all tasks. In other words $\mathcal{L}_i(\thetab^*) = \min_{\thetab} \mathcal{L}_i \ \forall i \in \mathcal{T}$, and hence $\nabla_{\thetab^*} \mathcal{L}_i = \mbf{0} \ \forall \ i \in \mathcal{T}$ (joint minimum). Furthermore, running gradient descent on $\min_{\thetab}\mathcal{L}^{\text{MT}}$ corresponds to the unitary scalarization ($\S$\ref{sec:setting}), which concludes the proof. %
\end{proof}

\randomgraddrop*

Proposition~\ref{prop:random-graddrop} can be proved by adapting the proof from \citet[proposition 1]{Chen2020}: it suffices to replace $f(\mathcal{P})$ with the Bernoulli parameter $p$, which is non-negative by definition. 
In our opinion, this seriously undermines the conflicting gradient hypothesis that motivated GradDrop.
For the reader's convenience, we now provide a straightforward and self-contained proof.

\begin{proof}
	Let us start from the statement on $\nabla_{\thetashb} \mathcal{L}^{\text{RGD}}$. If $\nabla_{\thetashb} \mathcal{L}_i = \mbf{0} \ \forall i \in \mathcal{T}$, then $\nabla_{\thetashb} \mathcal{L}^{\text{RGD}} = \mbf{0}$ with probability one. On the other hand, if $\exists j: \nabla_{\thetashb} \mathcal{L}_j \neq \mbf{0}$, then:
	\begin{equation*}
	\begin{aligned}
	\mathds{P}\left[ \nabla_{\thetashb} \mathcal{L}^{\text{RGD}} \neq \mbf{0} \right] &\geq \mathds{P}\left[ \nabla_{\thetashb} \mathcal{L}^{\text{RGD}} = \nabla_{\thetashb} \mathcal{L}_j \right] \\
	&= p (1 - p)^{m-1} > 0,
	\end{aligned}
	\end{equation*}
	where the first inequality comes from the fact that $\nabla_{\thetashb} \mathcal{L}^{\text{RGD}} = \nabla_{\thetashb} \mathcal{L}_j$ is only one of the many instances of a non-null $\nabla_{\thetashb} \mathcal{L}^{\text{RGD}}$.
\end{proof}

Let $\text{sign}(\xb)$ stand for the element-wise sign operator applied on $\xb$.
On encoder-decoder architectures, similarly to \gls{mgda} and \gls{imtl} (see appendices~\ref{sec:supp-mgda} and \ref{sec:supp-imtl}), the authors do not apply GradDrop on $\nabla_{\thetashb} \mathcal{L}_i$, but rather on a the usually less expensive $\nabla_{\zb} \mathcal{L}_i$. 
In more detail, they compute the GradDrop sign purity scores $\mbf{p}$ from equation \cref{eq:graddrop} on $\sum_{i=1}^{n}\left(\text{sign}(\zb) \odot \nabla_{\zb} \mathcal{L}_i\right)[i] \in \mathbb{R}^r$, and then apply equation \cref{eq:graddrop} on the $\nabla_{\zb} \mathcal{L}_i$ gradients, %
yielding a vector $\gb_{z} \in \mathbb{R}^{n \times r}$. 
Then, relying on reverse-mode differentiation,
the update direction in the space of the parameters $\thetashb$ is obtained via a Jacobian-vector product: $\gb = -\left( \frac{\partial \zb}{\partial \thetashb} \right)^T \gb_{z}$. Such a computation replaces the similar $\nabla_{\thetashb} \mathcal{L}^{\text{MT}} = \left(\frac{\partial \zb}{\partial \thetashb}\right)^T \nabla_{\zb} \mathcal{L}^{\text{MT}}$ from the unitary scalarization.
\newpage

\section{Experimental Setting, Reproducibility}

We now present details concerning the experimental settings from $\S$\ref{sec:experiments}, including details on the employed open-source software, dataset information, hardware specifications, and hyper-parameters.

\subsection{Supervised Learning} \label{sec:sl-setup}

All the experiments were run under Ubuntu 18.04 LTS, on a single GPU per run (using two 8-GPU machines in total). Timing experiments were all run on Nvidia GeForce GTX 1080 Ti GPUs, with an Intel Xeon E5-2650 CPU.
The remaining experiments were run on either Nvidia GeForce RTX 2080 Ti GPUs or Nvidia GeForce GTX 1080 Ti GPUs, respectively using an Intel Xeon Gold 6230 CPU or an Intel Xeon E5-2650 CPU.

\subsubsection{MultiMNIST}

Multi-MNIST, originally introduced by \citet{Sabour2017} and as modified by \citet{Sener2018}, is a simple two-task supervised learning benchmark dataset constructed by uniformly sampling MNIST~\citep{LeCun1998} images, and placing one in the top-left corner, the other in the bottom-right corner. Each of the two overlaid images corresponds to a 10-class classification task. 
Using the above procedure, we generate the Multi-MNIST training set from the first $50000$ MNIST training images, the validation set from the last $10000$ training images, and the test set from the original MNIST test set.
For consistency with the experimental setup of \citet{Sener2018}, we employ a modified encoder-decoder version of the LeNet architecture~\citep{LeCun1998}. Specifically, the last layer is omitted from the encoder, and two fully-connected layers are employed as task-specific predictive heads. The cross-entropy loss is used for both tasks.
All methods are trained for $100$ epochs using Adam~\citep{Kingma2015} in the stochastic gradient setting, with an initial learning rate of $\eta = 10^{-2}$ (tuned in $\eta \in \{10^{-3}, 10^{-2}, 10^{-1}\}$ and yielding the best validation results for all considered algorithms), exponentially decayed by $0.95$ after each epoch, and a mini-batch size of $256$.

\subsubsection{CelebA}

The CelebA~\citep{liu2015faceattributes} dataset consists of $200,000$ headshots (with standard training, validation and test splits) associated with the presence or absence of $40$ attributes. In the \gls{mtl} literature, is commonly treated as a $40$-task classification problem, each task being a binary classification problem for an attribute.
As commonly done in previous work~\citep{Sener2018,Yu2020,Liu2021}, we employ an encoder-decoder architecture where the encoder is a ResNet-18~\citep{He2016} (without the final layer) with batch normalization layers~\citep{Ioffe15}, and the per-task decoders are linear classifiers. The cross-entropy loss is used for all tasks.
The training is performed from scratch for $50$ epochs using Adam, with a mini-batch size of $128$ and a per-epoch exponential decay factor of $0.95$. As common on this network-dataset combination~\citep{rlw,Chen2020}, the initial learning rate is $\eta = 10^{-3}$ for all methods except for MGDA and IMTL, for which $\eta = 5 \times 10^{-4}$ yielded a better validation performance.
As done by the respective authors, for PCGrad, RLW and GradDrop we use the same learning rate as the unitary scalarization~\citep{Yu2020,rlw,Chen2020}. 

\subsubsection{Cityscapes}

We rely on the version of the dataset pre-processed by \citet{Liu2019}, which consists of $2,975$ training and $500$ test images and presents two tasks: semantic segmentation on $7$ classes, and depth estimation. We further split the original training set into a validation set of $595$ images, employed to tune hyper-parameters, and a training set of $2380$ images.
Consistently with recent work~\citep{rlw}, we rely on a dilated ResNet-50 architecture pre-trained on ImageNet~\citep{Yu2017} for the encoder, and on the Atrous Spatial Pyramid Pooling~\citep{ASPP}, which internally uses batch normalization, as decoders. 
While more powerful encoders might lead to better performance on Cityscapes, like the SegNet~\citep{SegNet} used in~\citep{javaloy2022rotograd,liu2021conflict,Navon2022}, we aim to provide a fair comparison of \gls{mtl} optimizers, rather than maximize overall task performance.
Cross-entropy loss is employed on the semantic segmentation task, whereas the $\ell_1$ loss is used for the depth estimation.
The training is performed by using Adam with a mini-batch size of $32$ for $100$ epochs, with an initial step size $\eta = 5 \times 10^{-4}$ resulting in the best validation performance for all algorithms, exponentially decayed by $0.95$ at each epoch.

\subsection{Reinforcement Learning}
\label{ref:rl-hyperparams}

Similarly to the supervised learning experiments, we ran all the experiments under Ubuntu 18.04 LTS using one GPU per run (using six 8-GPU machines in total).
Timing experiments were all run using NVIDIA GeForce RTX 2080 Ti GPUs, with an Intel Xeon Gold 6230 CPU.
The main bulk of the remaining experiments was run on Nvidia GeForce RTX 2080 Ti GPUs with either Intel Xeon Gold 6230 or Intel Xeon Silver 4216.
We utilised NVIDIA GeForce RTX 3080 GPUs with Intel Xeon Gold 6230 CPUs for a small fraction of experiments.

We use Meta-World's MT10/MT50 for our experiment.
The benchmark consists of ten/fifty tasks in which a simulated robot manipulator has to perform various actions, e.g., pressing a button, opening a door, or pushing the block.
We use~\citet{Sodhani2021} for most of the hyperparameters and list them in Table~\ref{tab:hyperparameters}.
We use bold font where we use a hyperparameter different from~\citet{Sodhani2021}.
Similarly to~\citet{Sodhani2021}, we use the \textsc{v1} version of Metaworld for our experiments\footnote{\url{https://github.com/rlworkgroup/metaworld.git@af8417bfc82a3e249b4b02156518d775f29eb289}}.
\citet{Sodhani2021} use a shared entropy loss weight $\alpha$ for PCGrad and separate $\alpha$ for unitary scalarization\footnote{\url{https://mtrl.readthedocs.io/en/latest/pages/tutorials/baseline.html}}.
In our experiments, use shared $\alpha$ for all of the methods for fairness.
Since it is a single number (rather than a vector), we used unitary scalarization to update $\alpha$ for all~\glspl{smto} apart from PCGrad which was already implemented in~\citep{Sodhani2021}.

We use the same network architecture as in~\citet{Sodhani2021}, i.e. a three-layered feedforward fully-connected network with 400 hidden units per layer for both, the actor and the critic. 
The actor is shared across all tasks as well as the critic.

To normalize rewards, we keep track of first and second moments in the buffer and normalise the rewards by their standard deviation: $r'_i = \nicefrac{r_i}{\hat{\sigma}_i},$ where $\hat{\sigma}_i$ is the sample standard deviation of the rewards for environment $i$.

\citet{Sodhani2021} average the gradient for unitary scalarisation and pcgrad, whereas our~\gls{smto} implementations sum the gradients, i.e. effectively using larger learning rates (apart from MGDA that assures that all the aggregation weights sum to 1). 
We tried reducing the learning rate for~\glsplural{smto} that sum (RLW Norm., RLW Diri., and GradDrop) both for MT10 and MT50, but it worked worse for these methods and we kept the default learning rate for them as well.
We had to use a smaller learning rate for IMTL, because with the default one it crashed at the beginning of training due to numerical overflow.
Smaller learning rate did not prevent it from crashing, but this happened much later.

\begin{wraptable}{L}{0.5\textwidth}
	\centering
	\caption{Hyperparameters of the~\gls{rl} experiments. Hyperparameters different from~\citet{Sodhani2021} are in bold.}
	\label{tab:hyperparameters}
	\begin{tabular}{ll}
		\textbf{Hyperparameter}&\textbf{Value}\\
		\midrule
		All methods &\\
		\midrule
		-- training steps & 2,000,000\\
		-- batch size & 1280\\
		-- \textbf{Replay buffer size} & \textbf{4,000,000}\\
		-- actor learning rate & 0.0003\\
		-- critic learning rate & 0.0003\\
		-- entropy $\alpha$ learning rate & 0.0003\\
		-- shared entropy $\alpha$ & True\\
		-- runs & 10\\
		-- discounting $\gamma$ & 0.99\\
		\midrule
		Unit. Scal.&\\
		\midrule
		-- \textbf{actor} $l_2$ \textbf{coeff.}& \textbf{0.0003}\\
		\midrule
		PCGrad&\\
		\midrule
		-- \textbf{actor} $l_2$ \textbf{coeff.}& \textbf{0.0001}\\
		\midrule
		RLW Norm.&\\
		\midrule
		-- normal mean & 0\\
		-- normal std & 1\\
		-- actor $l_2$ coeff.& 0.0003\\
		\midrule
		RLW Diri.&\\
		\midrule
		-- $\alpha$&1\\
		-- actor $l_2$ coeff.& 0.0003\\
		\midrule
		GradDrop&\\
		\midrule
		-- $$k$$&1\\
		-- $$p$$&0.5\\
		-- actor $l_2$ coeff.& 0.0001\\
		\midrule
		MGDA&\\
		\midrule
		-- gradient normalization & $L_2$\\
		-- actor $l_2$ coeff.& 0.0\\
		\midrule
		IMTL&\\
		-- actor learning rate & 0.00003\\
		-- critic learning rate & 0.00003\\
		-- entropy $\alpha$ learning rate & 0.00003\\
		-- actor $l_2$ coeff.& 0.0\\
		\midrule
	\end{tabular}
\vspace{-30pt}
\end{wraptable}

We tried $10^{6}$, $2\times10^{6}$, and $4\times10^{6}$ for the replay buffer size with the last being superior in terms of stability.
Additionally, for $l_2$ actor regularization, we tried $0.0001$ and $0.0003$ with the latter being slightly superior for the baseline.
We tried the same options for other~\glspl{smto}, and picked the best option for each of the method.
For MGDA, no regularisation works best, most likely due to a strong regularization effect of the method itself, which is mirrored by our supervised learning results.
PCGrad and Graddrop work best with the regularization coefficient of $0.0001$.
Both RLW variants use the same coefficient as the baseline ($0.0003$).

For MT50, we took the best MT10 hyperparameters, and we believe one could obtain even better results for unitary scalarisation since it is much faster to tune compared to other~\glspl{smto} (e.g. 15 hours for unitary scalarisation vs 9 days for PCGrad).

\subsection{Software Acknowledgments and Licenses} \label{sec:licenses}

Our codebase is built upon several prior works: \citep{Sener2018}, \cite{Liu2019}, \citep{rlw} and \citep{Sodhani2021}: all of them were released under a MIT license. We also acknowledge \citet{Pytorch-PCGrad}, upon which we built some of our code.
Multi-MNIST is based on MNIST dataset that is released under Creative Commons Attribution-Share Alike 3.0 license.
The code for generating Multi-MNIST dataset was taken from~\citet{Sener2018} released under MIT license.
CelebA dataset has a custom license allowing non-commercial research purposes.
More details can be found on the project website:\url{http://mmlab.ie.cuhk.edu.hk/projects/CelebA.html}.
Cityscapes also has a custom license allowing non-commercial research purposes. 
The full text of the license can be found on the project website:\url{https://www.cityscapes-dataset.com/license/}.
Metaworld, used for~\gls{rl} experiments is released under MIT license.

\section{Supplementary Supervised Learning Experiments} \label{sec:supp-sl-experiments}

This section presents supervised learning results omitted from $\S$\ref{sec:sl-experiments}. In particular, we show additional plots for the experiments of $\S$\ref{sec:sl-experiments}, then present an analysis of the regularising effect of \glspl{smto} in the absence of single-task regularization ($\S$\ref{sec:unreg}), and conclude with an ablation study on GradDrop's dependency on the sign of per-task gradients~($\S$\ref{sec:random-graddrop}).

\subsection{Addendum}

This section complements the plots presented in $\S$\ref{sec:sl-experiments}. In particular, we show the test and runtime results in table form, along with the behavior of the validation metrics and of the training loss over the training epochs.
Plots for Multi-MNIST, CelebA, and Cityscapes are reported in Figures \ref{fig:mnist-supplementary}, \ref{fig:celeba-supplementary} and \ref{fig:cityscapes-supplementary}, respectively.

The behavior of the CelebA training loss demonstrates heavier regularization (compare with the unregularized plot in Figure \ref{fig:celeba-unreg-loss}). Except IMTL and MGDA, for which the tuned values of the weight decay prevent overfitting, the other optimizers display very similar validation and training curves, and start overfitting around epoch $30$. Considering that most \glspl{smto} required less regularization (see $\S$\ref{sec:celeba-experiments}), the results are consistent with our interpretation of \glspl{smto} as regularizers in $\S$\ref{sec:critical-analysis}.
The Cityscapes plots display a certain instability across training epochs, as demonstrated by the various peaks and valleys in the metrics. Nevertheless, in spite of a factor $10$ difference in scale, both training losses are similarly decreased by most optimizers.
\vspace{10pt}
\begin{figure}[h!]
	\begin{subfigure}{0.49\textwidth}
		\centering
		\includegraphics[width=\textwidth]{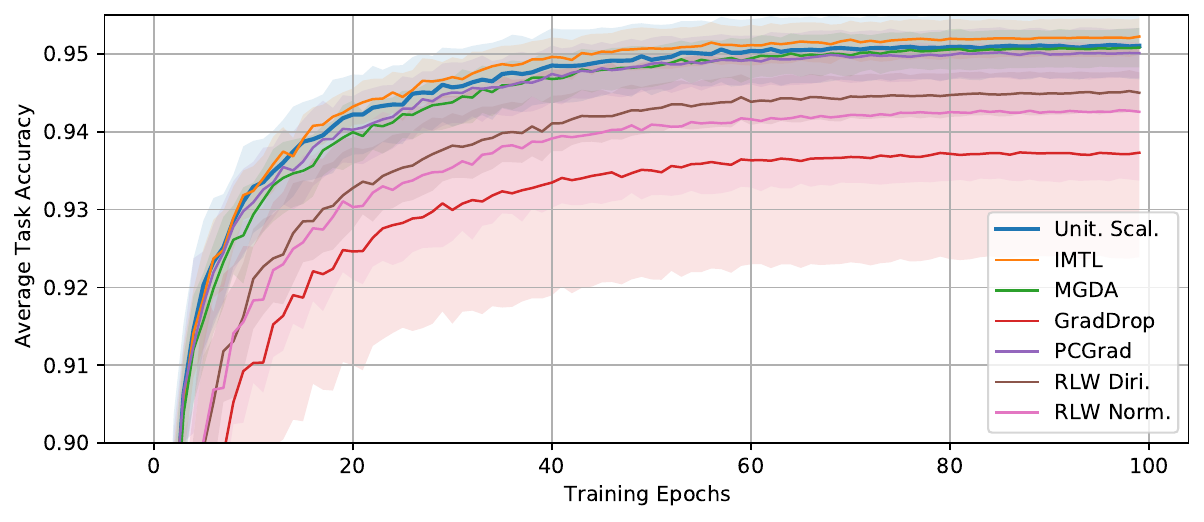}
		\caption{Mean (and 95$\%$ CI) average task validation accuracy per training~epoch.}
		\label{fig:mnist-plot}
	\end{subfigure}
	\begin{subfigure}{0.49\textwidth}
		\centering
		\includegraphics[width=\textwidth]{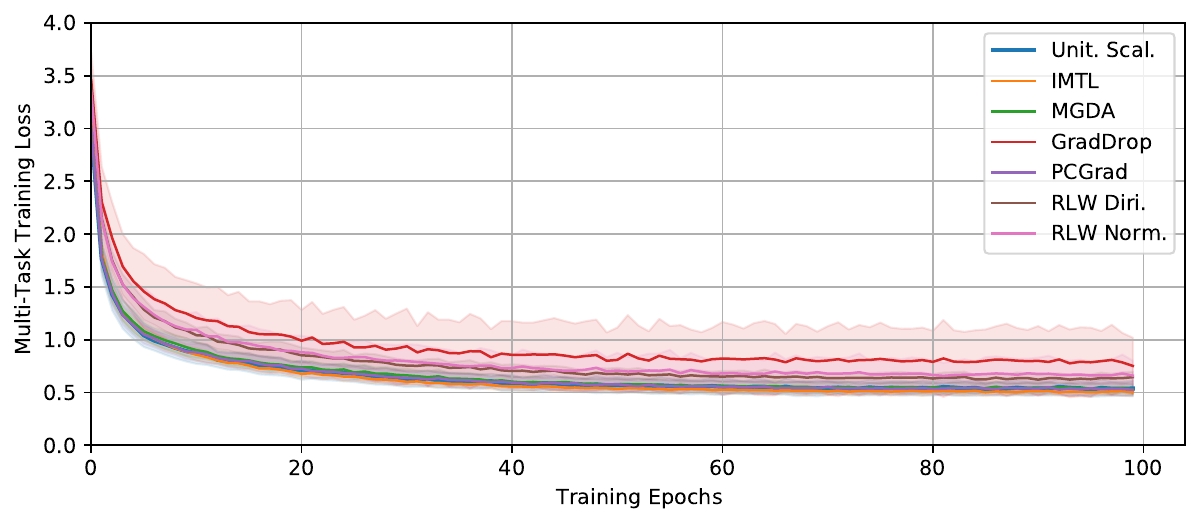}
		\caption{Mean (and 95$\%$ CI) training multi-task loss $\mathcal{L}^{\text{MT}}$ per epoch.}
		\label{fig:mnist-loss}
	\end{subfigure} 
	\sisetup{detect-weight=true,detect-inline-weight=math,detect-mode=true}
	\begin{subtable}{.49\textwidth}
		\centering
		\vspace{10pt}
		\footnotesize
		\begin{adjustbox}{max width=\textwidth, center}
			\begin{tabular}{lll}
				\toprule
				MTO &      Average Task Accuracy &       Epoch Runtime [s] \\
				\midrule
				Unit. Scal. &  9.476e-01 $\pm$ 4.368e-03 &  [3.510e+00, 3.617e+00] \\
				IMTL &  9.487e-01 $\pm$ 2.533e-03 &  [3.695e+00, 3.996e+00] \\
				MGDA &  9.478e-01 $\pm$ 1.977e-03 &  [3.491e+00, 3.617e+00] \\
				GradDrop &  9.347e-01 $\pm$ 1.282e-02 &  [3.508e+00, 3.589e+00] \\
				PCGrad &  9.479e-01 $\pm$ 3.578e-03 &  [3.807e+00, 3.928e+00] \\
				RLW Diri. &  9.430e-01 $\pm$ 2.973e-03 &  [3.790e+00, 4.005e+00] \\
				RLW Norm. &  9.399e-01 $\pm$ 8.929e-03 &  [3.894e+00, 4.225e+00] \\
				\bottomrule
			\end{tabular}
		\end{adjustbox}
		\subcaption{\label{table:mnist} \small Mean and 95$\%$ CI of the avg. task test accuracy across runs, and interquartile range for the training time per epoch.}
	\end{subtable} 	
	\hspace{10pt}
	\vspace{5pt}
	\begin{minipage}{.45\textwidth}
	\caption{\label{fig:mnist-supplementary} Additional figures for the comparison of various \gls{smto}s with the unitary scalarization on the MultiMNIST dataset~\citep{Sener2018}. }
	\end{minipage}
\end{figure}

\begin{figure}[h!]
	\begin{subfigure}{0.49\textwidth}
		\centering
		\includegraphics[width=\textwidth]{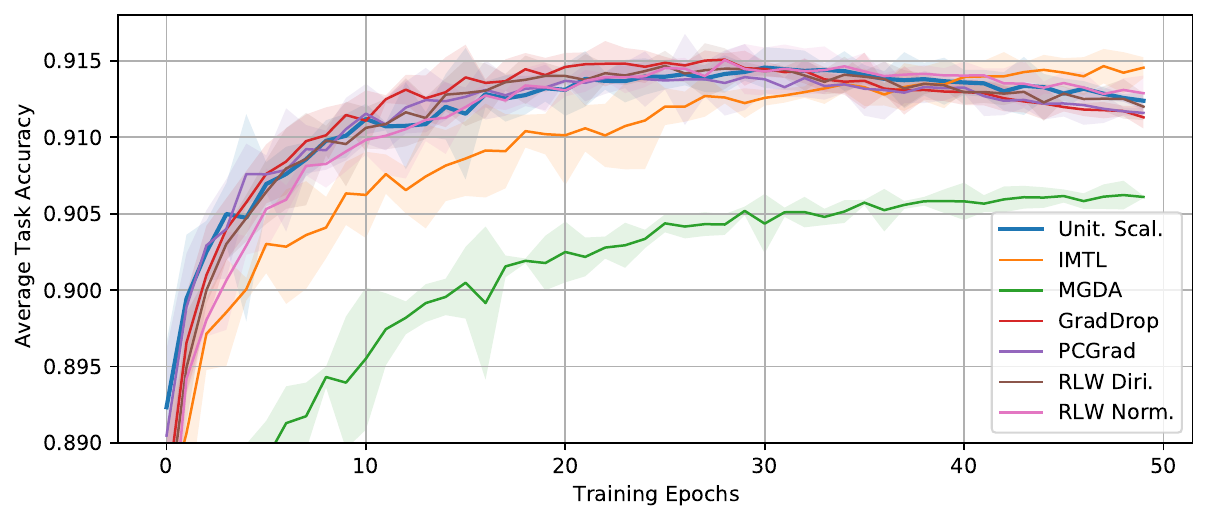}
		\caption{Mean (and 95$\%$ CI) average task validation accuracy per training epoch.}
		\label{fig:celeba-plot}
	\end{subfigure}
	\begin{subfigure}{0.49\textwidth}
		\centering
		\includegraphics[width=\textwidth]{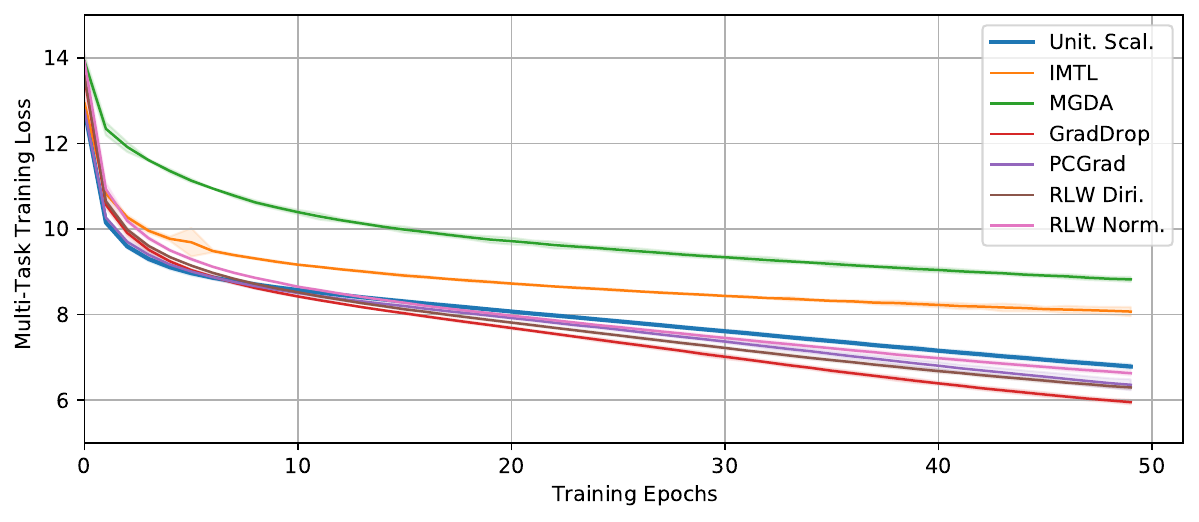}
		\caption{Mean (and 95$\%$ CI) training multi-task loss $\mathcal{L}^{\text{MT}}$ per epoch.}
		\label{fig:celeba-loss}
	\end{subfigure}
	\sisetup{detect-weight=true,detect-inline-weight=math,detect-mode=true}
	\begin{subtable}{.49\textwidth}
		\centering
		\vspace{10pt}
		\footnotesize
		\begin{adjustbox}{max width=\textwidth, center}
			\begin{tabular}{lll}
				\toprule
				MTO &     Average Task Accuracy &      Epoch Runtime [s] \\
				\midrule
				Unit. Scal. & 9.090e-01 $\pm$ 7.568e-04 & [2.869e+02, 2.878e+02] \\
				IMTL & 9.093e-01 $\pm$ 7.631e-04 & [3.600e+02, 3.621e+02] \\
				MGDA & 9.022e-01 $\pm$ 9.687e-04 & [6.859e+02, 7.194e+02] \\
				GradDrop & 9.098e-01 $\pm$ 3.383e-04 & [3.001e+02, 3.008e+02] \\
				PCGrad & 9.093e-01 $\pm$ 1.108e-03 & [1.015e+04, 1.016e+04] \\
				RLW Diri. & 9.099e-01 $\pm$ 7.845e-04 & [3.040e+02, 3.054e+02] \\
				RLW Norm. & 9.095e-01 $\pm$ 1.012e-03 & [3.028e+02, 3.037e+02] \\
				\bottomrule
			\end{tabular}
		\end{adjustbox}
		\subcaption{\label{table:celeba} \small Mean and 95$\%$ CI of the avg. task test accuracy across runs, and interquartile range for the training time per epoch.}
	\end{subtable} 
	\hspace{10pt}
	\vspace{5pt}
	\begin{minipage}{.45\textwidth}
	\caption{\label{fig:celeba-supplementary} Additional figures for the comparison of various \gls{smto}s with the unitary scalarization on the CelebA~\citep{liu2015faceattributes} dataset. }
	\end{minipage}

\end{figure}

\subsection{Unregularized Experiments} \label{sec:unreg}

Figures \ref{fig:celeba-unreg-validation}, \ref{fig:celeba-unreg-loss} and \ref{fig:celeba-unreg-val-loss} respectively report the average task validation accuracy, the multi-task training loss, and the multi-task validation loss at each training epoch.
The regularizing effect of \glspl{smto} compared to unitary scalarization is shown by: (i) the delay of the onset of overfitting on the validation data in figure~\ref{fig:celeba-unreg-validation}, (ii) the reduction of the convergence rate on the training loss in figure~\ref{fig:celeba-unreg-loss} (compare with figure~\ref{fig:celeba-loss}), and (iii) the fact that validation and training losses remain positively correlated for larger numbers of epochs.
In fact, the behavior of both the training and validation loss for the \glspl{smto} closely parallels that of $\ell_2$-regularized unitary scalarization, with differing degrees of regularization.
We further note that unregularized IMTL displays a certain instability (compare with the regularized version in figure~\ref{fig:celeba-plot}).

The addition of dropout layers further reduces overfitting, improves stability (reduced confidence intervals) and pushes the average validation curve upwards, motivating its use on all optimizers for the experiments of~$\S$\ref{sec:celeba-experiments}. Nevertheless, confidence intervals in Figure~\ref{fig:celeba-unreg-validation} still overlap due to the instability of the unregularized unitary scalarization. 
Figure~\ref{fig:celeba-unreg-violin} provides a more detailed comparison over $20$ repetitions, confirming that the combined use of dropout layers and $\ell_2$ regularization improves average performance and reduces the empirical variance for unitary scalarization.
Furthermore, Figure~\ref{fig:celeba-unreg-val-all} shows that regularization improves the peak average validation performance for all algorithms, demonstrating the need of tuning $\lambda$ also for \glspl{smto}.
We conclude by pointing out that even without regularization, when carefully tuned, the maximal performance over epochs of unitary scalarization is comparable to \glspl{smto} in Figure~\ref{fig:celeba-unreg-validation}.

\begin{figure}[!h]
	\begin{subfigure}{0.49\columnwidth}
		\centering
		\includegraphics[width=\columnwidth]{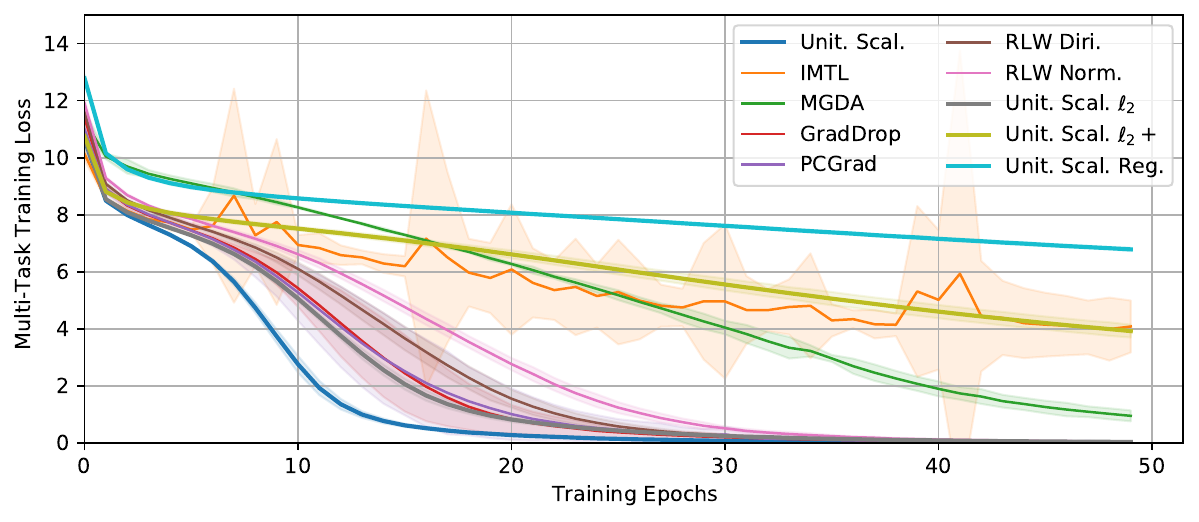}
		\vspace{-15pt}
		\caption{Mean and 95$\%$ CI (3 runs) multi-task training loss per epoch. }
		\label{fig:celeba-unreg-loss}
	\end{subfigure}
	\begin{subfigure}{0.49\columnwidth}
		\centering
		\includegraphics[width=\columnwidth]{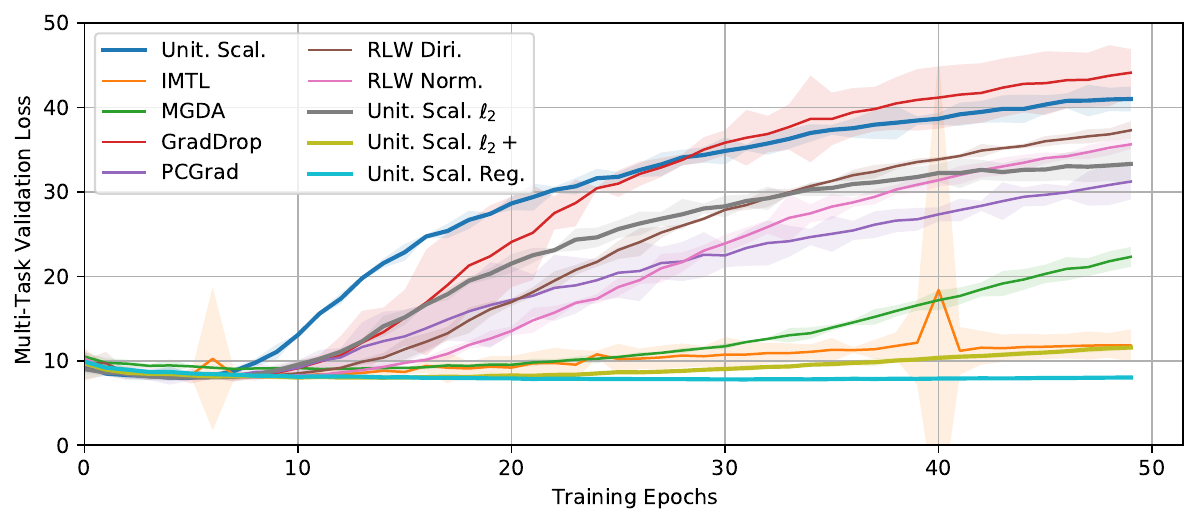}
		\vspace{-15pt}
		\caption{Mean and 95$\%$ CI (3 runs) multi-task validation loss per training~epoch. }
		\label{fig:celeba-unreg-val-loss}
	\end{subfigure}
	\caption{Additional figures for the unregularized comparison of various \gls{smto}s with the unitary scalarization on CelebA. \label{fig:celeba-unreg}
	\glspl{smto} provide varying degrees of regularization.}
\end{figure}
\begin{figure}[!h]
	\centering
	\includegraphics[width=0.49\columnwidth]{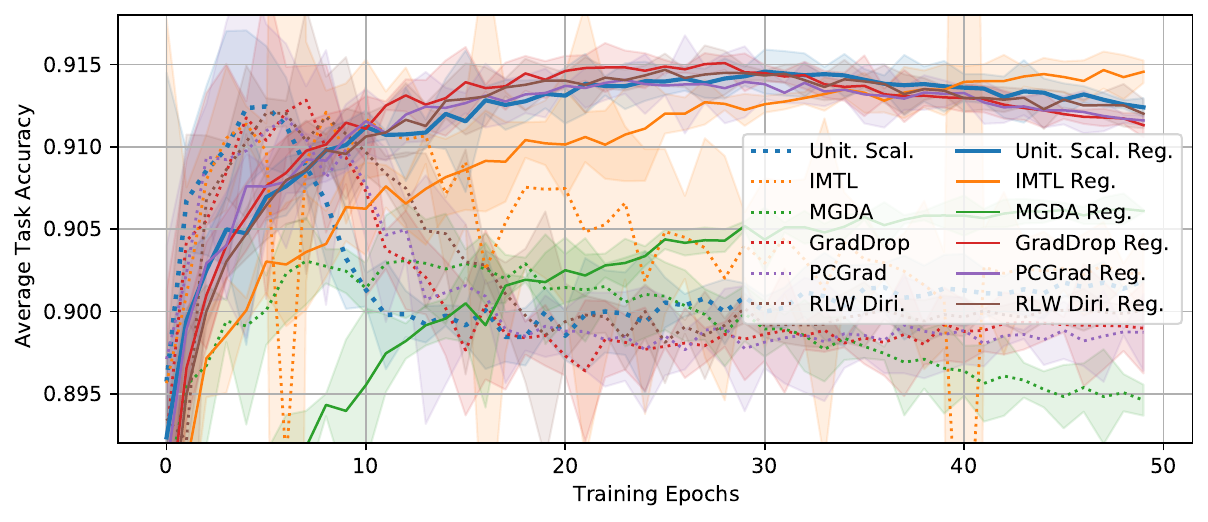}
	\caption{
		Effect of regularization (dropout layers and weight decay) on the average task validation accuracy for all considered optimizers on the CelebA dataset: 
		regularization improves the average performance of all algorithms.
		\label{fig:celeba-unreg-val-all}}
\end{figure}
\begin{figure}[!h]
		\centering
		\includegraphics[width=0.49\columnwidth]{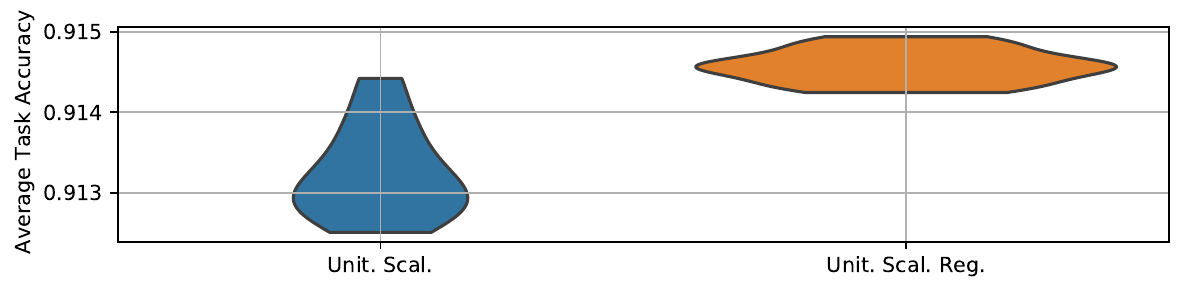}
	\caption{
		Effect of regularization (dropout layers and weight decay) on unitary scalarization on the CelebA dataset: 
		violin plots (20 runs) for the best avg. task validation accuracy over epochs. The width at a given value represents the proportion of runs yielding that result.
		\label{fig:celeba-unreg-violin}
		Regularization improves the average performance while decreasing its variability.}
\end{figure}

\subsection{Sign-Agnostic GradDrop} \label{sec:random-graddrop}

We will now present an ablation study on GradDrop, investigating the effect of the sign of per-task gradients on the \gls{smto}'s performance.
Specifically, we compare the performance of GradDrop with a sign-agnostic version of its stochastic gradient masking (which we refer to as ``Sign-Agnostic GradDrop"), whose update direction is defined as follows:
\begin{equation*}
	\gb = -\left( \frac{\partial \zb}{\partial \thetashb} \right)^T \left(\sum_{i \in \mathcal{T}} \mbf{u}_i \odot \nabla_{\zb} \mathcal{L}_i\right),
\end{equation*}
where $\mbf{u}_i, \nabla_{\zb} \mathcal{L}_i \in \mathbb{R}^{n \times r}$ and, for all $i \in \mathcal{T}$, $\mbf{u}_i$ is i.i.d. according to $\mbf{u}_i[j, k] \sim \text{Bernoulli}(p) \ \forall j \in \{1, \dots, n\}, k \in \{1, \dots, r\}$.
Differently from a similar study carried out by \citet{Chen2020}, we tuned the hyper-parameter of the sign-agnostic masking in the following range: $p\in\{0.1, 0.25, 0.5, 0.75, 0.9\}$.

The experimental setup complies with the one described in appendix \ref{sec:sl-setup}.
Figure~\ref{fig:celeba-graddrop}, plotting test and validation results for the CelebA dataset~\citep{liu2015faceattributes}, shows that the performance of Sign-Agnostic GradDrop closely matches the original algorithm. Therefore, sign conflicts across per-task gradients do not seem to play a significant role in GradDrop's performance.

\begin{figure}[h!]
	\begin{subfigure}{0.49\textwidth}
		\centering
		\includegraphics[width=\textwidth]{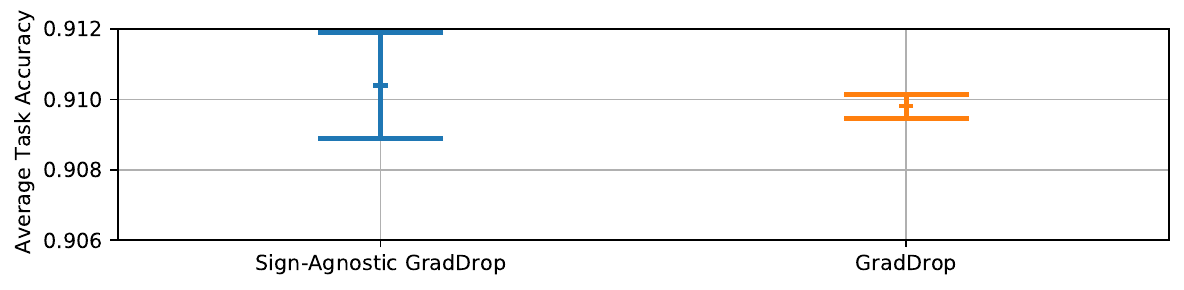}
		\caption{Mean and 95$\%$ CI (3 runs) avg. task test accuracy.}
		\label{fig:celeba-graddrop-best}
	\end{subfigure} %
	\begin{subfigure}{0.49\textwidth}
		\centering
		\includegraphics[width=\textwidth]{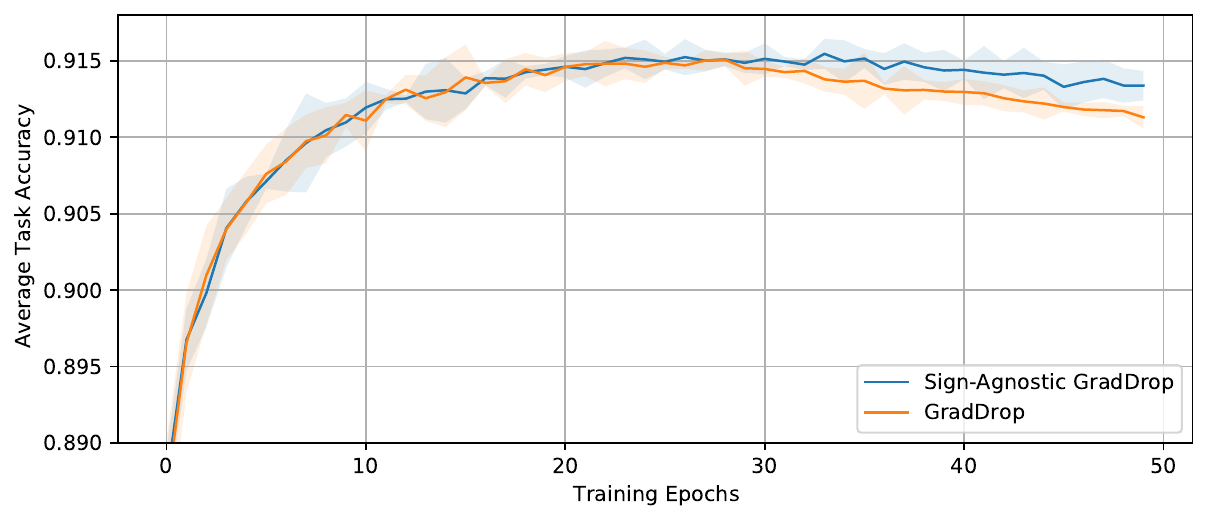}
		\caption{Mean and 95$\%$ CI (3 runs) avg. task validation accuracy per training epoch.}
		\label{fig:celeba-graddrop-epoch}
	\end{subfigure}
	\caption{\label{fig:celeba-graddrop} Comparison of GradDrop~\citep{Chen2020} with sign-agnostic masking of the shared-representation gradients on the CelebA dataset~\citep{liu2015faceattributes}. No statistically relevant difference between the two methods can be observed for the majority of the epochs.}
\end{figure}

\section{Supplementary Reinforcement Learning Experiments} \label{sec:supp-rl-experiments}

\subsection{Addendum}

This section presents additional plots for the \gls{rl} experiments in $\S$\ref{sec:rl-experiments}. Specifically, Figure~\ref{fig:mt-with-imtl} re-plots Figure~\ref{fig:mt10-success} and \ref{fig:mt50-success} with the omitted IMTL results, while Figure~\ref{fig:mt-curves} shows the learning curves omitted from $\S$\ref{sec:rl-experiments}. 
As pointed out in~$\S$\ref{sec:rl-experiments}, none of the IMTL runs successfully terminated due to numerical instability. 
Indeed, \citet{Liu2021} show that, in supervised settings, coefficients do not fluctuate much across epochs~\citep[Figure 4, appendix B]{Liu2021} and never become negative.  By contrast, up to 50\% of the scaling coefficients $\alpha$ are negative in our experiments, thus reversing subtask gradient directions.
MGDA, which constrains the weights, is more stable and is comparable to unitary scalarization.
In order to avoid incomplete curves and unfair calculations of the mean, Figure~\ref{fig:mt-curves} plots the highest value ever achieved by \emph{any} seed as a dashed horizontal line. The IMTL results in Figure~\ref{fig:mt-with-imtl}, instead, report the best average success rate of each seed until its termination.

\begin{figure*}[t!]
	\begin{subfigure}{0.49\textwidth}
		\centering
		\includegraphics[width=\textwidth]{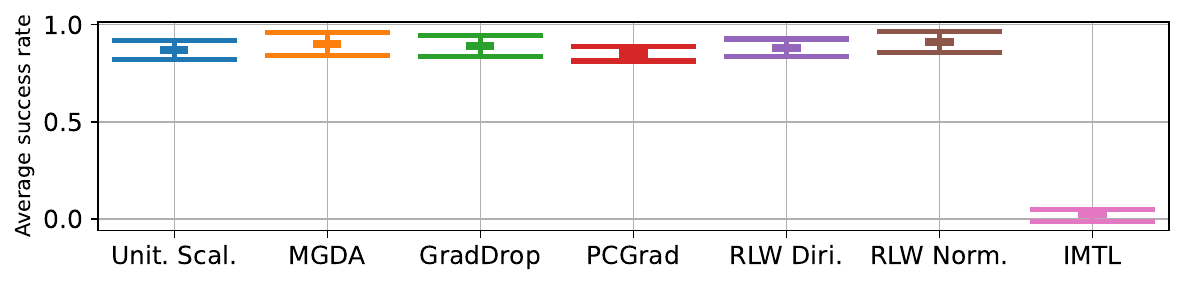}
		\vspace{-18pt}
		\caption{\label{fig:mt10-success-with-imtl} MT10 (10 runs per method).}
	\end{subfigure}\hspace{5pt}
	\begin{subfigure}{0.49\textwidth}
		\centering
		\includegraphics[width=\columnwidth]{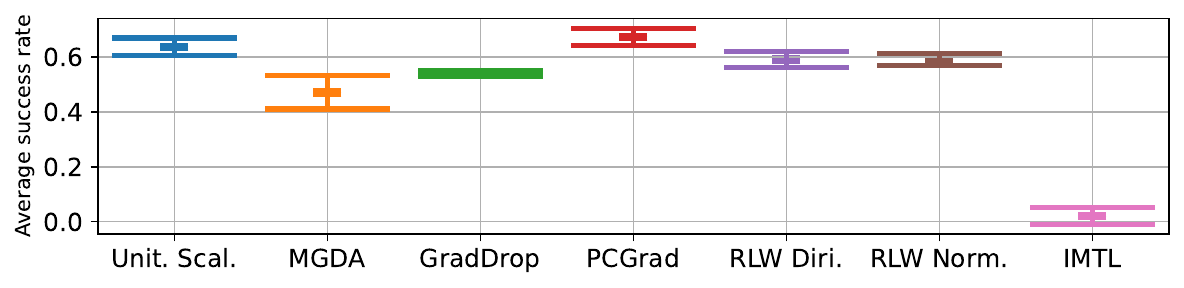}
		\vspace{-18pt}
		\caption{\label{fig:mt50-success-with-imtl} MT50 (10 runs per method).}
	\end{subfigure}
	\vspace{-5pt}
	\caption{\label{fig:mt-with-imtl} Mean and 95$\%$ CI for the best avg. success rate on Metaworld. None of the \glspl{smto} significantly outperforms unitary scalarization.}
\end{figure*}
\begin{figure*}[t!]
	\begin{subfigure}{0.49\textwidth}
		\centering
		\includegraphics[width=\textwidth]{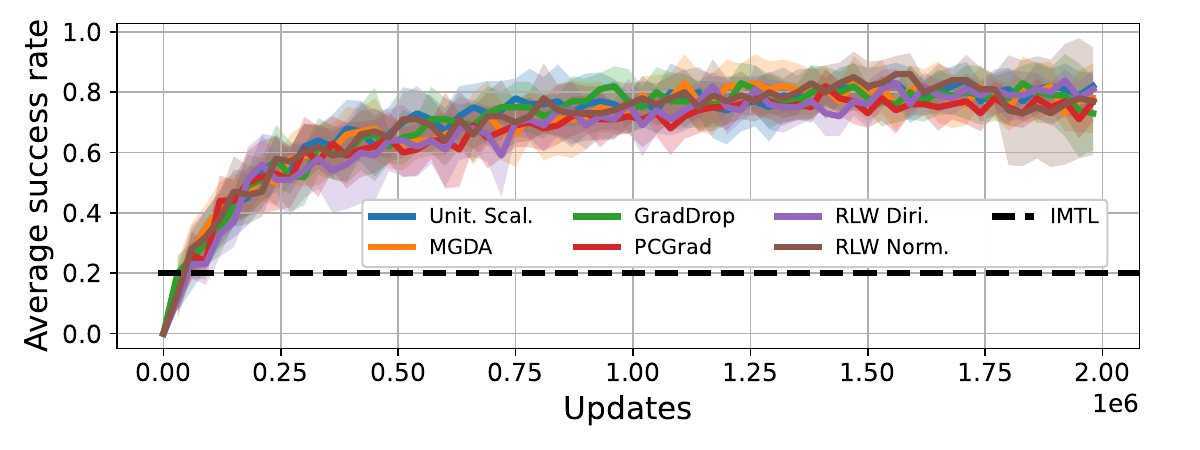}
		\vspace{-15pt}
		\caption{\label{fig:mt10-curves} MT10 (10 points per method).}
	\end{subfigure}\hspace{5pt}
	\begin{subfigure}{0.49\textwidth}
		\centering
		\includegraphics[width=\textwidth]{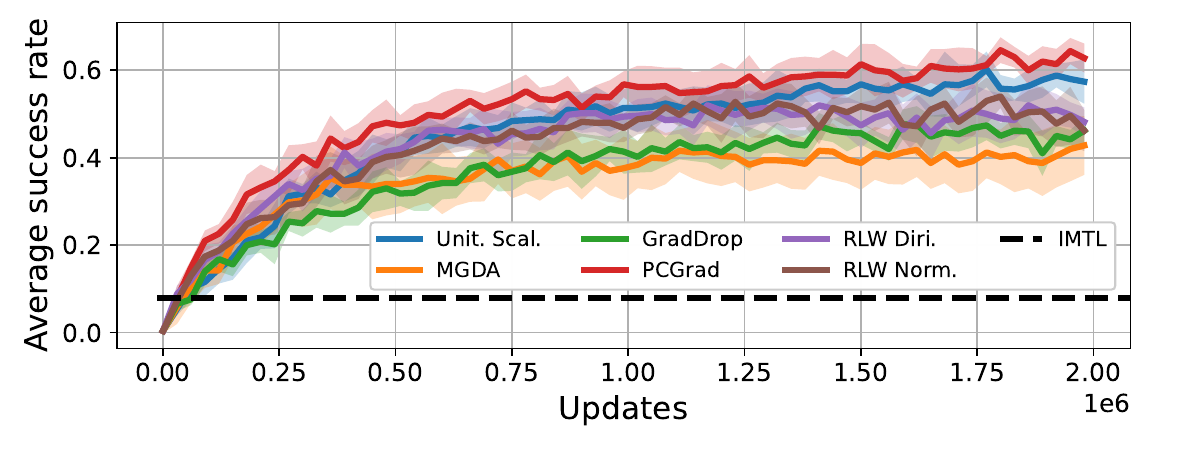}
		\vspace{-15pt}
		\caption{\label{fig:mt50-curve} MT50 (10 points per method).}
	\end{subfigure}
	\vspace{-5pt}
	\caption{\label{fig:mt-curves} Mean and 95$\%$ CI for the avg. success rate on Metaworld. None of the \glspl{smto} significantly outperforms unitary scalarization.}
\end{figure*}

\subsection{Ablation studies}
\label{sec:rl-ablations}

Figure~\ref{fig:mt10-ablations-shared-alpha} presents our ablations for MT10 experiments.
Due to computational constraints, we ran ablations on the unitary scalarization and PCGrad since these are the two methods previously tested in the~\gls{rl} setting.

Figure~\ref{fig:rl-reg} shows ablation studies on the effect of regularization on MT10 and MT50. In spite of CI overlaps, actor $l_2$ regularization pushes the average higher on both benchmarks, motivating our use of regularization for the experiments in $\S$\ref{sec:rl-experiments}. Furthermore, the gap between the averages tends to widen with the number of updates on MT50, suggesting improved stabilization.

\begin{figure}[t!]
	\begin{subfigure}{0.49\textwidth}
		\centering
		\includegraphics[width=\columnwidth]{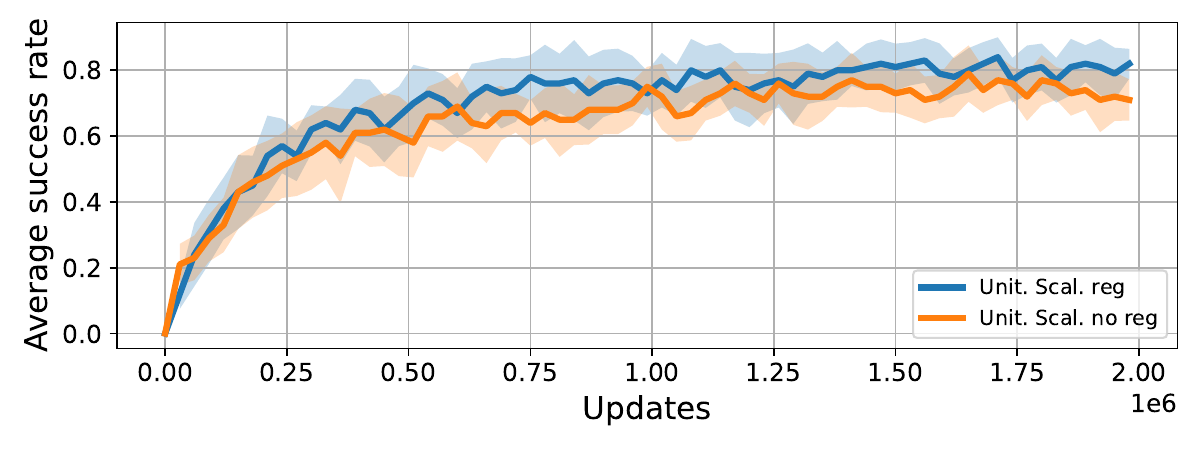}
		\vspace{-15pt}
		\caption{MT10 average performance (10 runs) and 95\% CI.}
	\end{subfigure}\hspace{5pt}
	\begin{subfigure}{0.49\textwidth}
		\centering
		\includegraphics[width=\columnwidth]{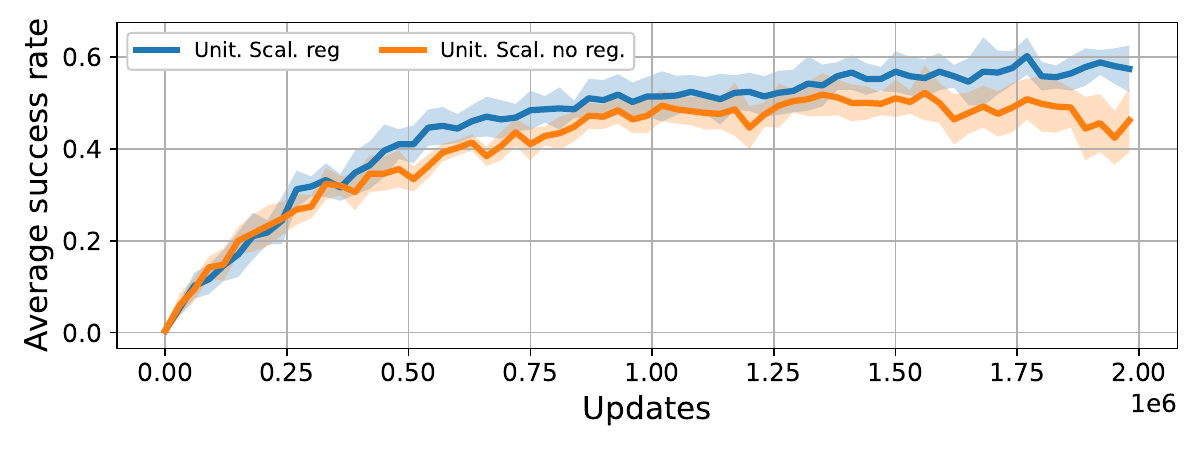}
		\vspace{-15pt}
		\caption{MT50 average performance (5 runs) and 95\% CI.}
	\end{subfigure}
	\vspace{-3pt}
	\caption{\label{fig:rl-reg} For both MT10 and MT50, actor $l_2$ regularization pushes the average higher for unitary scalarization.}
\end{figure}

\subsection{Sensitivity to Reward Normalization}

Figure~\ref{fig:rewnorm-sensitivity} shows that multitask agent performance is highly sensitive to the reward normalization moving average hyperparameter\footnote{\tiny \url{https://github.com/facebookresearch/mtenv/blob/4a6d9d6fdfb321f1b51f890ef36b5161359e972d/mtenv/envs/metaworld/wrappers/normalized\_env.py\#L69}} motivating our buffer normalization in Section~\ref{sec:rl-experiments}.

\begin{figure}[h!]
	\centering
	\includegraphics[width=.49\columnwidth]{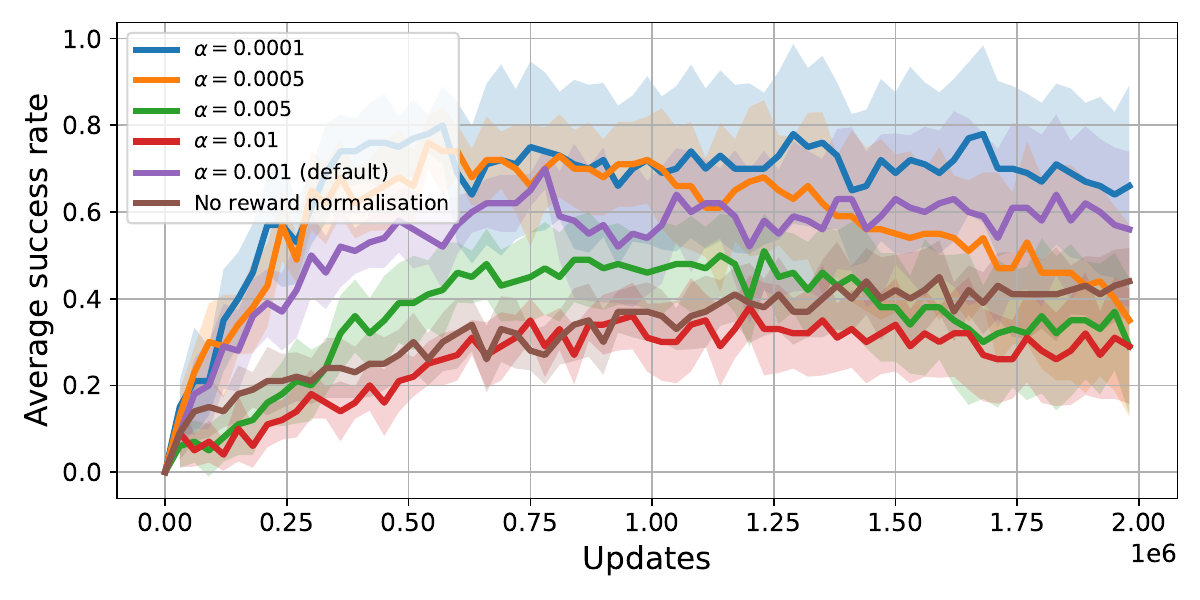}
	\caption{The learning outcomes of a Multitask SAC agent vary considerably depending on the reward normalisation hyperparameter. Each of the curves represents and average of 10 runs with shaded 95\% confidence interval.}
	\label{fig:rewnorm-sensitivity}
\end{figure}

\begin{figure*}[h!]
	
	\begin{subtable}{\textwidth}
		\label{table:cityscapes-supplementary}
		\caption{\small Mean and 95$\%$ CI of the test metrics across runs, and interquartile range for the training time per epoch.}
		\sisetup{detect-weight=true,detect-inline-weight=math,detect-mode=true}
		\centering
		\vspace{5pt}
		\footnotesize
		\begin{adjustbox}{max width=\textwidth, center}
			\begin{tabular}{llllll}
				\toprule
				MTO &      Absolute Depth Error &      Relative Depth Error &     Segmentation Accuracy &         Segmentation mIOU &      Epoch Runtime [s] \\
				\midrule
				Unit. Scal. & 1.301e-02 $\pm$ 2.342e-04 & 4.761e+01 $\pm$ 5.148e+00 & 9.196e-01 $\pm$ 2.913e-04 & 7.012e-01 $\pm$ 6.001e-04 & [3.228e+02, 3.241e+02] \\
				IMTL & 1.281e-02 $\pm$ 7.521e-04 & 4.389e+01 $\pm$ 6.984e-01 & 9.164e-01 $\pm$ 2.828e-03 & 6.967e-01 $\pm$ 4.785e-03 & [7.329e+02, 7.373e+02] \\
				MGDA & 1.418e-02 $\pm$ 2.331e-04 & 4.750e+01 $\pm$ 1.466e+01 & 9.189e-01 $\pm$ 2.636e-04 & 6.999e-01 $\pm$ 3.124e-03 & [7.251e+02, 7.269e+02] \\
				GradDrop & 1.293e-02 $\pm$ 2.757e-04 & 4.674e+01 $\pm$ 7.709e+00 & 9.193e-01 $\pm$ 1.282e-03 & 7.024e-01 $\pm$ 3.628e-03 & [5.196e+02, 5.215e+02] \\
				PCGrad & 1.294e-02 $\pm$ 2.284e-04 & 4.380e+01 $\pm$ 5.165e+00 & 9.198e-01 $\pm$ 9.119e-04 & 7.025e-01 $\pm$ 6.531e-04 & [4.202e+02, 4.212e+02] \\
				RLW Diri. & 1.305e-02 $\pm$ 4.155e-04 & 4.810e+01 $\pm$ 2.259e+00 & 9.199e-01 $\pm$ 1.247e-03 & 7.037e-01 $\pm$ 1.989e-03 & [3.161e+02, 3.164e+02] \\
				RLW Norm. & 1.301e-02 $\pm$ 5.528e-04 & 4.630e+01 $\pm$ 2.751e+00 & 9.192e-01 $\pm$ 4.962e-04 & 7.006e-01 $\pm$ 4.580e-03 & [3.194e+02, 3.210e+02] \\
				\bottomrule
			\end{tabular}
		\end{adjustbox}
	\end{subtable} 
	\vspace{10pt}
	
	\begin{subfigure}{0.49\textwidth}
		\centering
		\includegraphics[width=\textwidth]{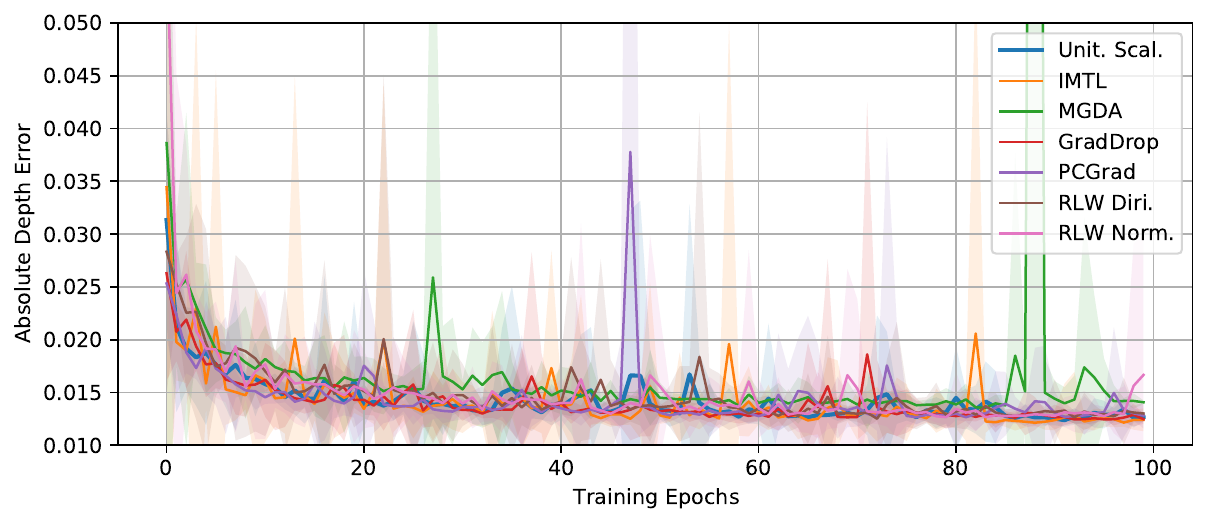}
		\caption{Mean (and 95$\%$ CI) absolute depth validation error per training epoch.}
		\label{fig:cityscapes-plot-absdepth}
	\end{subfigure} \hspace{3pt}
	\begin{subfigure}{0.49\textwidth}
		\centering
		\includegraphics[width=\textwidth]{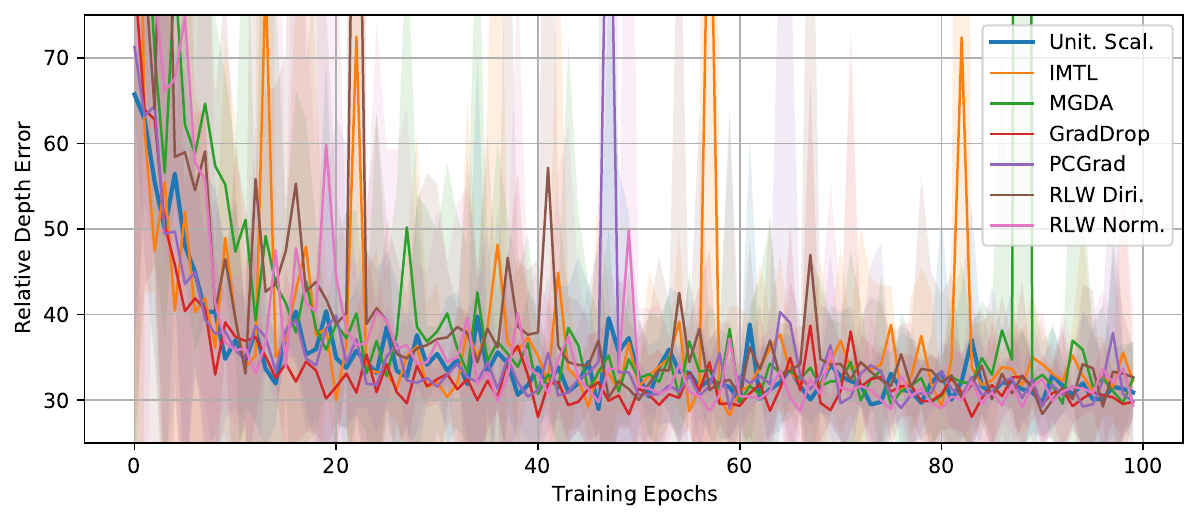}
		\caption{Mean (and 95$\%$ CI) relative depth validation error per training epoch.}
		\label{fig:cityscapes-plot-reldepth}
	\end{subfigure}
	\begin{subfigure}{0.49\textwidth}
		\centering
		\includegraphics[width=\textwidth]{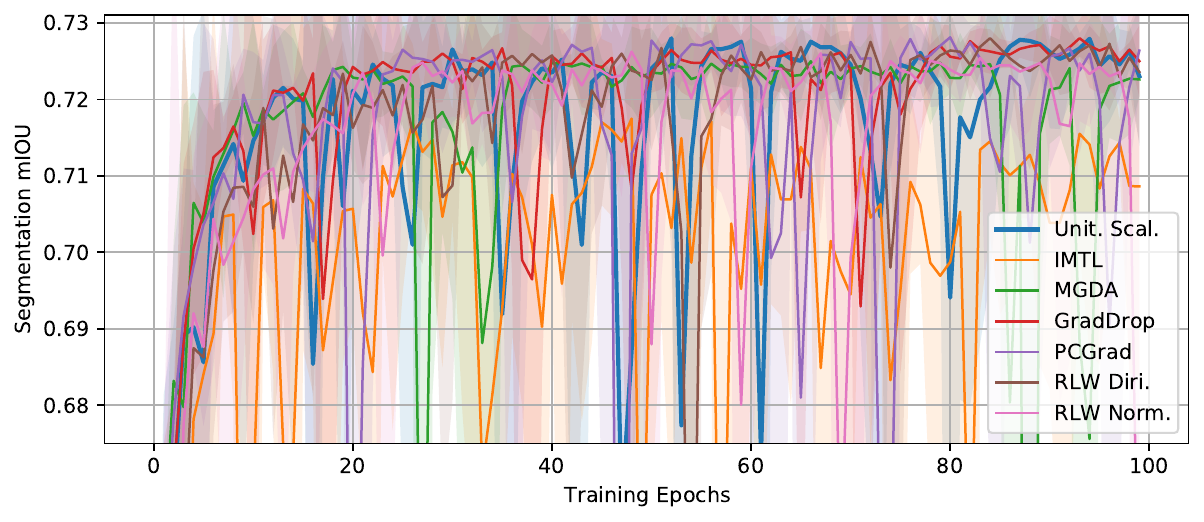}
		\caption{Mean (and 95$\%$ CI) validation segmentation mIOU per training epoch.}
		\label{fig:cityscapes-plot-segmIOU}
	\end{subfigure} \hspace{3pt}
	\begin{subfigure}{0.49\textwidth}
		\centering
		\includegraphics[width=\textwidth]{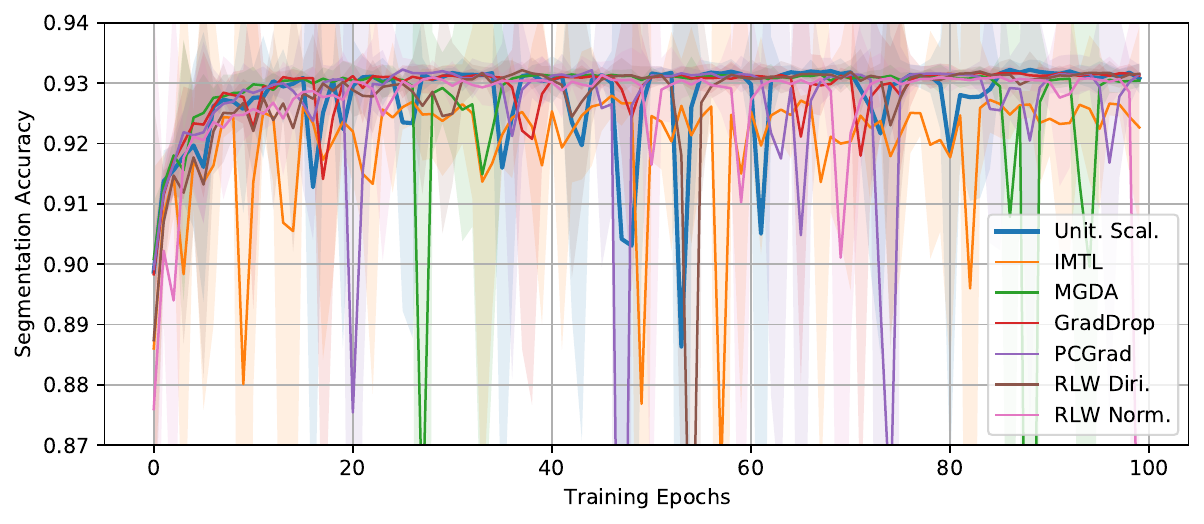}
		\caption{Mean (and 95$\%$ CI) validation segmentation accuracy per training epoch.}
		\label{fig:cityscapes-plot-segacc}
	\end{subfigure}
	\begin{subfigure}{0.49\textwidth}
		\centering
		\includegraphics[width=\textwidth]{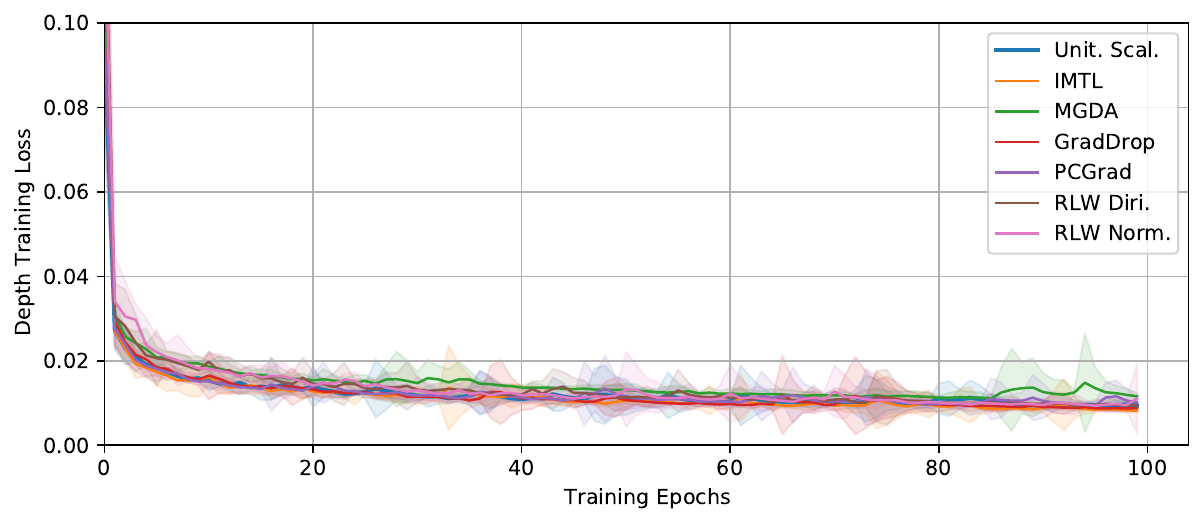}
		\caption{Mean (and 95$\%$ CI) training depth loss per epoch.}
		\label{fig:cityscapes-plot-deploss}
	\end{subfigure} \hspace{3pt}
	\begin{subfigure}{0.49\textwidth}
		\centering
		\includegraphics[width=\textwidth]{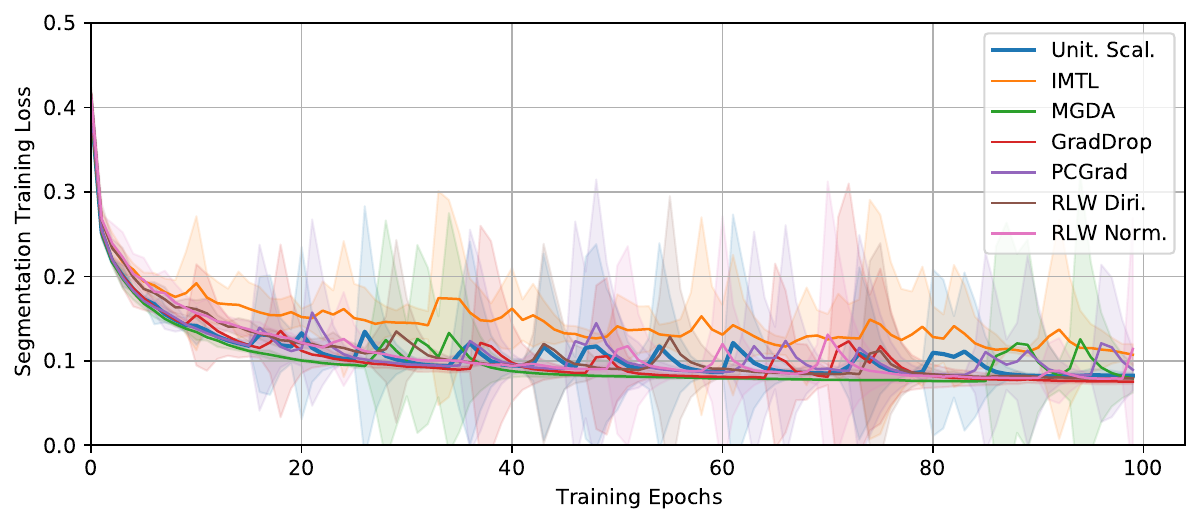}
		\caption{Mean (and 95$\%$ CI) training segmentation loss per epoch.}
		\label{fig:cityscapes-plot-segloss}
	\end{subfigure}
	\caption{\small Additional figures for the comparison of \gls{smto}s with the unitary scalarization on the Cityscapes~\citep{Cityscapes} dataset. \label{fig:cityscapes-supplementary}}
	
\end{figure*}

\begin{figure*}[h!]
	\centering
	\includegraphics[width=0.9\textwidth]{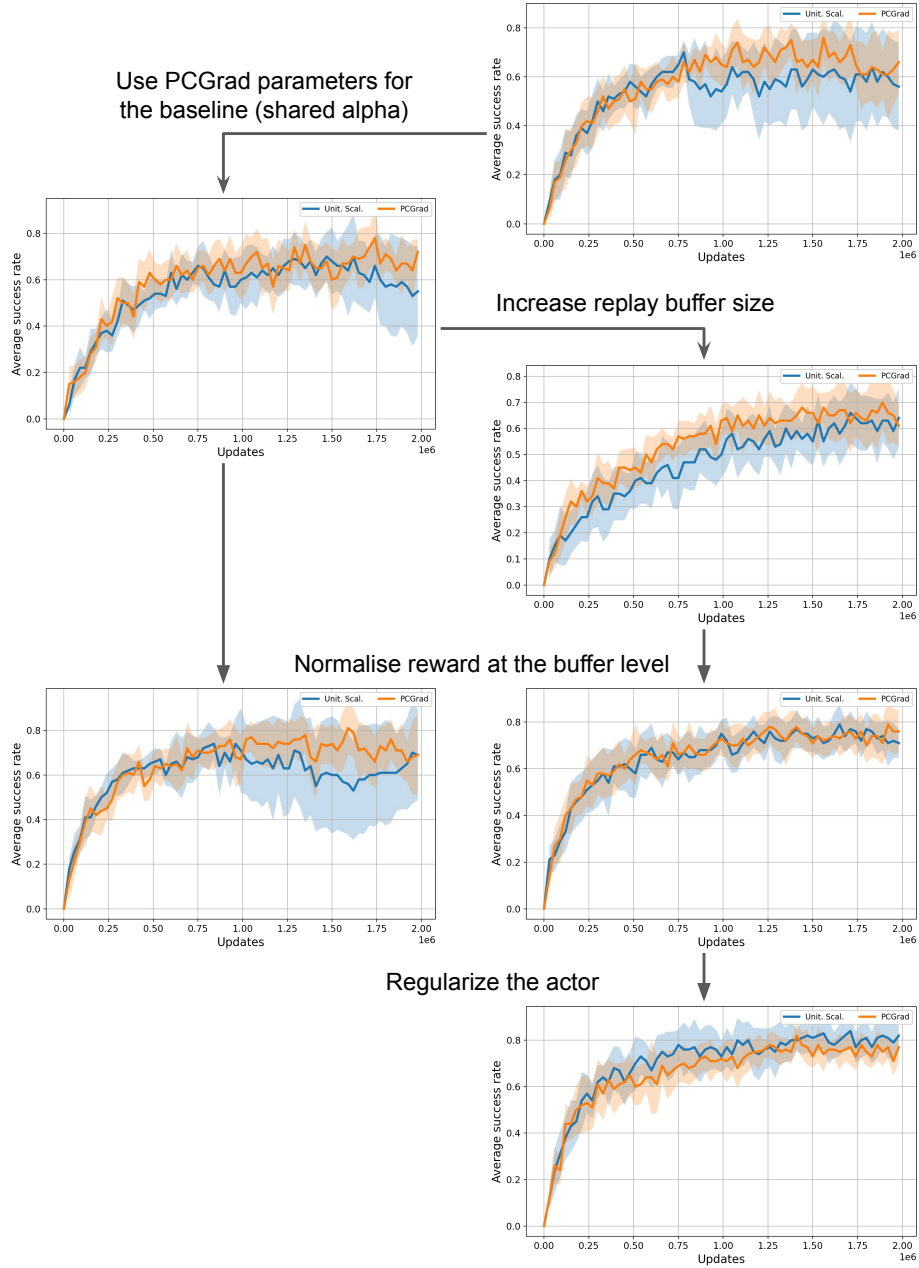}
	\caption{Metaworld's MT10 ablation experiments.}
	\label{fig:mt10-ablations-shared-alpha}
\end{figure*}

}

\end{document}